\documentclass[twoside]{article}
\usepackage{aistats2025}

\usepackage[utf8]{inputenc} % allow utf-8 input
\usepackage[T1]{fontenc}    % use 8-bit T1 fonts
\usepackage{hyperref}       % hyperlinks
\usepackage{url}            % simple URL typesetting
\usepackage{booktabs}       % professional-quality tables
\usepackage{amsfonts}       % blackboard math symbols
\usepackage{nicefrac}       % compact symbols for 1/2, etc.
\usepackage{microtype}      % microtypography
\usepackage{xcolor}         % colors
\colorlet{darkgreen}{green!45!black}
\hypersetup{
  colorlinks,
  citecolor=darkgreen}
\usepackage{hyphenat}

\usepackage{adjustbox}
\usepackage{amsfonts}
\usepackage{xcolor}

\usepackage{pifont}
\newcommand{\xmark}{\ding{55}}
\newcommand{\vmark}{\textcolor{green!50!black}{\ding{51}}}
\usepackage{graphicx}
\usepackage{subfigure}
\usepackage{enumitem}

\usepackage{algorithmic}
\usepackage{algorithm}

\usepackage{amsmath, amssymb, mathtools, amsthm}
\usepackage{thmtools, thm-restate}

\usepackage[capitalize,noabbrev]{cleveref}

\theoremstyle{plain}
\declaretheorem{theorem}
\declaretheorem[sibling=theorem, numberwithin=section]{proposition}
\declaretheorem[sibling=theorem, numberwithin=section]{lemma}
\declaretheorem[sibling=theorem, numberwithin=section]{corollary}
\theoremstyle{definition}
\declaretheorem[sibling=theorem, numberwithin=section]{definition}
\declaretheorem[sibling=theorem, numberwithin=section]{assumption}
\theoremstyle{remark}

\newcommand{\defeq}{\overset{\text{\tiny def}}{=}}
\newcommand\indep{\protect\mathpalette{\protect\indepT}{\perp}}
\def\indepT#1#2{\mathrel{\rlap{$#1#2$}\mkern2mu{#1#2}}}

\usepackage[normalem]{ulem}
\usepackage{bbold}
\usepackage{bigints}
\usepackage{stmaryrd}
\usepackage{booktabs,tabularx}
\usepackage{makecell}
\usepackage[nobiblatex]{xurl}
\newcommand\rurl[1]{%
  \href{http://#1}{\nolinkurl{#1}}%
}

\DeclareMathOperator*{\argmax}{arg\,max}
\DeclareMathOperator*{\argmin}{arg\,min}

\usepackage{tikz}
\usetikzlibrary{backgrounds}
\usetikzlibrary{tikzmark}
\usetikzlibrary{calc}
\usetikzlibrary{arrows,shapes,positioning,shadows,trees,mindmap}
\usetikzlibrary{arrows.meta}

\usepackage{tcolorbox}
\newcommand{\highlight}[2]{\colorbox{#1!17}{$\displaystyle #2$}}

\renewcommand{\highlight}[2]{\colorbox{#1!17}{#2}}

\newtcolorbox{CatchyBox}[2][]{
    lower separated=false,
    colback=white,
    colframe=white, fonttitle=\bfseries,
    colbacktitle=white,
    coltitle=black,
    title=#2,#1}

\usepackage{snaptodo}
\snaptodoset{block rise=2em}
\snaptodoset{margin block/.style={font=\scriptsize}} 
\usepackage{natbib}

\begin{document}

\twocolumn[
\aistatstitle{Survival Models: Proper Scoring Rule \\ and Stochastic Optimization with Competing Risks}

\aistatsauthor{Julie Alberge$^1$, Vincent Maladière$^2$, Olivier Grisel$^2$, Judith Abécassis$^1$, Gaël Varoquaux$^1$}

\aistatsaddress{$^1$ SODA Team, Inria Saclay, Palaiseau France \\
$^2$ :probabl., Paris France\\
\texttt{julie.alberge@inria.fr}, \texttt{vincent@probabl.ai}
} ]

\begin{abstract}

When dealing with right-censored data, where some outcomes are missing due to a limited observation period, survival analysis —known as \textit{time-to-event analysis}— focuses on predicting the time until an event of interest occurs. Multiple classes of outcomes lead to a classification variant: predicting the most likely event, a less explored area known as \textit{competing risks}. Classic competing risks models couple architecture and loss, limiting scalability.

To address these issues, we design a strictly proper censoring-adjusted separable scoring rule, allowing optimization on a subset of the data as each observation is evaluated independently. The loss estimates outcome probabilities and enables stochastic optimization for competing risks, which we use for efficient gradient boosting trees. \textbf{\textsc{SurvivalBoost}} not only outperforms 12 state-of-the-art models across several metrics on 4 real-life datasets, both in competing risks and survival settings, but also provides great calibration, the ability to predict across any time horizon, and computation times faster than existing methods.

\end{abstract}

\section{INTRODUCTION}

We all die at some point. Some applications call for predicting not \emph{if} but \emph{when} an event of interest is likely to occur. In such a setting of \emph{time-to-event regression}, samples often have unobserved outcomes, \emph{e.g.} individuals that have not been followed long enough for the event of interest to occur. Limiting the analysis to fully observed samples creates a censoring bias. To address this, \emph{survival analysis} models use dedicated corrections for censorship. These have long been central to health applications \citep{zhu_deep_2016, chaddad_radiomic_2016, gaynor_use_1993}. Nowadays, survival analysis is also used in diverse fields, such as predictive maintenance \citep{rith_analysis_2018, susto_machine_2015}, or user-engagement studies \citep{maystre_temporally-consistent_nodate}. 
Survival analysis has led to many dedicated models, such as the \citet{kaplan_nonparametric_1958} estimator or the \citet{cox_regression_1972} proportional hazard model. 

Competing risks analysis generalizes survival analysis to multiple events, determining which will happen first \citep{susto_machine_2015, gaynor_use_1993}. For instance, if a breast-cancer patient dies from a different cause, it is impossible to determine when they would have succumbed to cancer, regardless of the duration of the observation period. The caregiver may also want to adapt the treatment if it is predicted that the patient will die of a competing event, such as a heart attack, sooner than from cancer.
As the risks of the various events are seldom independent--for example, cancer and cardiovascular disease share inflammation or age risk factors \citep{koene2016shared}--competing risks cannot be solved by running a survival model for each event \citep{wolbers2009prognostic}. The estimated risk of an event of interest will be biased if the competing risks are not included. Hence, adequate models for these risks are critical for decision-making~\citep{ramspek2022lessons,koller2012competing,van2016competing}.

Survival models have traditionally been developed with \emph{ad hoc} adjustments for censoring. The most common approach is to design a likelihood using the probability of censoring per unit time--\emph{i.e.} the time-derivative of the risk--which either comes with strong parametric assumptions \citep{cox_regression_1972} or \emph{ad hoc} corrections \citep{wang_survtrace_2022}. Given that the risk, which is the probability of the outcome at a specific time, is crucial for various applications, it is preferable to use proper scoring rules, that directly control probabilities, as developed by \citet{graf_assessment_1999,rindt_survival_2022}. However, no metric (or loss) has been shown to control probabilities in the competing risks setting.

In application domains typical of survival analysis and competing risks --health, predictive maintenance, insurance, marketing-- the data are mostly tabular with categorical variables, where tree-based models shine \citep{grinsztajn_why_2022}.  Existing survival and competing risks models do not fit well with these requirements. In particular, the proper scoring rule in \citet{rindt_survival_2022} requires a time derivative of the risk, typically via an auto-diff operator in a neural architecture. This approach is challenging to adapt to tree-based algorithms. In addition, the ever-growing volume of data calls for computationally efficient algorithms.

\paragraph{Contributions}
Here, we provide a general theoretical framework for learning a competing risks algorithm using a strictly proper scoring rule. This scoring rule yields a loss function easy to plug into any multiclass estimator to create a competing risks algorithm, providing the individual risk of each event at any given horizon.
% We also sum over time for model evaluation, as the resulting Integrated Scoring Rule remains proper.\\ 
An interesting property of this new loss is that it can be optimized on a subset of the training data. Hence, it allows stochastic optimization, enabling computationally efficient learning. \\
With that, we propose an algorithm called \textsc{SurvivalBoost}, based on Stochastic Gradient Boosting Trees. 
We benchmark our algorithm on a synthetic dataset and 4 real-world datasets - both in the competing risks and the survival analysis setting - with several ranking and calibration metrics and show that it outperforms 12 state-of-the-art (SOTA) baselines in both settings.

\section{RELATED WORK}
\paragraph{Survival settings}
Various survival models have been developed, ranging from approaches like the \citet{kaplan_nonparametric_1958} estimator, which estimates the general survival curve for an entire population, to models that account for covariates. The \citet{cox_regression_1972} Proportional Hazards Model, a linear model of the \emph{hazards}, which represents the instantaneous probability of an event, \emph{i.e.}, the logarithmic derivative of outcome probabilities over time. More complex models have been adapted to the survival setting: Support Vector Machines \citep{van_belle_support_2011}, survival games \citep{han2021inverse} and Neural networks with DeepSurv \citep{katzman_deepsurv_2018} or PCHazard \citep{kvamme_continuous_2019}. While these models do not control risks,  more recent neural networks employ appropriate loss functions: DQS \citep[though relying on a piecewise constant hazard]{yanagisawa2023proper}, SumoNet \citep{rindt_survival_2022} which requires differentiable models.

\paragraph{Competing risks}
Competing risks, involving multiple possible outcomes, require new methods that can naturally adapt to the simpler survival analysis setting.
Derived from the \citet{kaplan_nonparametric_1958} estimator, the \citet{nelson_theory_1972}-\citet{aalen_survival_2008} estimator is an unbiased marginal model for competing risks. \\
The linear \citet{fine_proportional_1999} estimator, inspired by the \citet{cox_regression_1972} estimator in survival analysis, is the most popular model in clinical research.
Recently, machine learning models have been adapted to competing risks settings, including tree-based approaches such as the Random Survival Forests \citep{ishwaran_random_2008, kretowska_tree-based_2018, bellot_tree-based_2018}, boosting approaches \citep{bellot_multitask_2018}, and neural networks approaches such as DeepHit and Gaussian mixtures approaches \citep{lee_deephit_2018, aala_deep_2017, danks_derivative-based_2022, nagpal2021deep}. Tranformer-based approaches with SurvTRACE \citep{wang_survtrace_2022} using a loss corrected to predict rare competing events, independently forecasts all events but do not ensure that probabilities sum to one.\\
For a comprehensive review of competing risks models, refer to \citet{monterrubio-gomez_review_2022}.

\paragraph{Evaluation for such models}
Prediction evaluation in survival or competing risks settings requires adapted metrics to account for right-censored data points \citep{harrell}, such as the C-index, which is an adaptation of the Area Under the ROC Curve (AUC) used in classification tasks. However, the C-index only evaluates the ranking of samples, \emph{i.e.} which samples are likely to experience the event of interest first. It is also dependent on the censoring distribution, which can introduce bias in the evaluation~\citep{blanche_c-index_2019}. In fact, the score may be inflated for distributions that differ from the oracle-censoring distribution\cite{rindt_survival_2022}. Alternative methods have been proposed, such as the \emph{time-dependent} C-index, $C_\zeta$ \citep{antolini_timedependent_2005}, which is the same metric but computed at a specific time horizon $\zeta$. The C-index ranking metric has also been extended to competing risks~\citep{uno_cstatistics_2011}, but, as in the survival setting, it only evaluates relative risks between pairs of individuals and does not assess the absolute risk for a given individual. 
Other time-dependent adaptations of the ROC curve have been developed, though these also measure discriminative power rather than the actual risks or probabilities ~\citep{blanche2013estimating}.
Yet, controlling risk is crucial for decision making \citep{van2019calibration}. Proper scoring rules offer an alternative to overcome the limitations of existing metrics, as they capture more aspects of the problem. Additionally, they can be used for both the training and evaluating probabilistic predictive models.

\paragraph{Proper Scoring Rules (PSR)}
Scoring rules are cost functions of observations and a candidate probability distribution. When \emph{proper}, they target the oracle probability distribution (Definition \ref{def:proper}).
Crucially, they give machine-learning losses that recover probabilities of outcomes.
For classification, where discrete events are observed rather than probabilities, the Brier score and the log loss give proper scoring rules, with relative merits \citep{benedetti_scoring_2010,merkle_choosing_2013}. 

\citet{graf_assessment_1999} adapt the Brier score to survival analysis, with a strong assumption of independence of the covariates in the censoring distribution. Yet, this assumption is often violated \citep{kvamme_brier_2019}, leading to bias \citep{rindt_survival_2022}.
\citet{rindt_survival_2022} show that the likelihood of the survival function yields a proper scoring rule, but requires both the density function and the survival function, which is a time-wise derivative of outcome probabilities (Definition \ref{def:proper}).  For quantile regression, \citet{yanagisawa2023proper} adapt the Pinball loss to a proper scoring rule for survival analysis, but requiring an oracle parameter.
\citet{han2021inverse} introduce a double optimization problem, where the stationary point corresponds to the oracle distributions.

For competing risks, \citet{schoop_quantifying_2011} extend the Brier score to a proper scoring rule. However, the Brier score does not capture the uncertainty as effectively as the log loss \citep{benedetti_scoring_2010}. 

\section{PROBLEM FORMULATION}
\paragraph{Notations}
We write oracle quantities as $a^*$ and estimates as $\hat{a}$, vectors in bold, $\mathbf{a}$, random variables in upper case, $A$, observations in lower cases $a$, and distributions in calligraphic style $\mathcal{A}$.
\subsection{Problem Setting}

We consider $K \in \mathbb{N}^*$ competing events. For $k \in \llbracket 1, K \rrbracket$, we denote $T^*_k \in \mathbb{R}_+$ the event time of the event $k$, depending on the covariates $\mathbf{X} \sim \mathcal{X}$. We also denote $T^* \in \mathbb{R}_+$, the first event of interest that occurs, $T^* = \min\limits_{k \in \llbracket 1, K \rrbracket}(T^*_k)$.
We observe $(\mathbf{X}, T, \Delta) \sim \mathcal{D}$, with $T = \min(T^*, C)$ where $C \in \mathbb{R}_+$ is the censoring time, which can depend on $\mathbf{X}$, and $\Delta \in \llbracket0, K \rrbracket, \Delta= \argmin\limits_{k \in \llbracket 0, K \rrbracket}(T^*_k) $, where 0 denotes a censored observation. 
However, we are primarily interested in the distribution of the uncensored data, $(\mathbf{X}, T^*, \Delta) \sim \mathcal{D}^*$, particularly the joint distribution of $T^*, \Delta|\mathbf{X} =\mathbf{x}$. \\
Given a data set of $n$ individuals, we denote each individual $i$ by its associated covariates $\textbf{x}_i$. The outcome is represented by $(t_i, \delta_i)$, where $t_i$ is the observed time, and $\delta_i \in \llbracket0, K\rrbracket $ is the event indicator. $\delta_i = k$ indicates that the event of interest $k$ was observed at time $t_i$, while $\delta_i = 0$ indicates that the observation was censored at time $t_i$. This paper aims to predict an unbiased estimate of all cause-specific Cumulative Incidence functions (CIFs) at any time horizon $\zeta$ (Definition \ref{quantities_interest}).

\begin{definition}[\emph{Quantities of interest}]\label{quantities_interest}~\\[-4.8ex]
\begin{flalign*}%
    \text{\parbox{\linewidth}{\footnotesize{Survival Function to any event}:}}\\[-.2em] 
    S^*(\zeta|\mathbf{x}) = \mathbb{P}(T^* > \zeta |\mathbf{X}= \mathbf{x})\\
    \text{\parbox{\linewidth}{\footnotesize{CIF} \small (Cumulative Incidence
    Function):}}\\[-.1em] F^*(\zeta|\mathbf{x}) = \mathbb{P}(T^* \leq \zeta
|\mathbf{X}= \mathbf{x}) = 1 - S^*(\zeta | \mathbf{X} = \mathbf{x})\\[.1em]
    \text{\parbox{\linewidth}{\footnotesize{CIF of the $k^{th}$ event}:}}\\ F^*_k(\zeta|\mathbf{x}) = \mathbb{P}(T^* \leq \zeta \cap \Delta = k |\mathbf{X}= \mathbf{x})  \\[.1em]
   \text{\parbox{\linewidth}{\footnotesize{Censoring Function}:}}\\[-.5em] G^*(\zeta|\mathbf{x}) = \mathbb{P}(C> \zeta |\mathbf{X}= \mathbf{x})
\end{flalign*}%
\end{definition}%
\vspace{-0.8em}
\begin{assumption}[\emph{Non-informative censoring}]\label{info_censoring}
     We make the classic assumption in survival analysis that censoring is non-informative with respect to covariates: $$ T^* \indep C \, | \, \mathbf{X}$$
\end{assumption}
Assumption~\ref{info_censoring} is essential for most theoretical results in survival analysis \citep{rindt_survival_2022, yanagisawa2023proper, han2021inverse}. It shows that single-event survival analysis becomes invalid in the presence of competing risks: if some observations are censored due to other events that share unobserved risk factors with the event of interest, this assumption is violated.

\subsection{CIF Scoring Rule}
\paragraph{Proper Scoring Rule }
A scoring rule $\ell$ evaluates a distribution $\mathcal{P}$ on an observation $Y$, producing a corresponding score $\ell(\mathcal{P}, Y)$. The higher the score, the better the model fits the observation. For a proper scoring rule, the score reflects the model's ability to predict the oracle distribution
\citep[for more on scoring rules, see][]{gneiting_strictly_2007, ovcharov_proper_2018, merkle_choosing_2013}. 
\begin{definition}[\emph{Proper Scoring Rule}]\label{def:proper}
A scoring rule $\ell$ is considered proper if
$$
\forall \mathcal{P}, \mathcal{Q}, \text{distributions} \quad 
\mathbb{E}_{Y\sim \mathcal{Q}}[\ell(\mathcal{P}, Y)]  \leq \mathbb{E}_{Y\sim \mathcal{Q}}[\ell(\mathcal{Q}, Y)].
$$
If the equality holds if and only if $\mathcal{P} = \mathcal{Q}$, in which case the scoring rule is \emph{strictly proper}.%
\end{definition} 
\paragraph{Proper scoring rule for the Global CIF}
We denote $L_{\zeta}$ a scoring rule for the global CIF at time $\zeta$. 
\begin{definition}[\emph{PSR for competing risks settings}]\label{def:psrcr}
In competing risks settings, where censoring is present, a scoring rule $L_{\zeta}$ for the CIF at time $\zeta$ for an observation $(\mathbf{X}, T, \Delta)$ is proper if and only if: \\
$ \forall \zeta, (\mathbf{X}, T, \Delta )\sim \mathcal{D}, \forall (\hat{F}_1, ..., \hat{F}_K, \hat{S}), $
%\vspace{0.5em}
\begin{equation}
\small
    \begin{aligned}
    \mathbb{E}_{T^*\!, C, \Delta  | \mathbf{X}= \mathbf{x}}
    [L_{\zeta}(
    \tikzmarknode{estimands}{\highlight{pink}
        {$(\hat{F}_1(\zeta| \mathbf{x}), ..., \hat{F}_K(\zeta| \mathbf{x}), \hat{S}(\zeta| \mathbf{x}))$}, (T, \Delta)
        }
    )]  \\
        \leq  \quad \\ 
    \mathbb{E}_{T^*\!, C, \Delta  | \mathbf{X}= \mathbf{x}}[L_{\zeta}(\tikzmarknode{oracle}{\highlight{purple}{$(F^*_1(\zeta| \mathbf{x}), ..., F^*_K(\zeta| \mathbf{x}), S^*(\zeta| \mathbf{x}))$}, (T, \Delta)})] 
    \end{aligned}%
\begin{tikzpicture}[overlay,remember picture,>=stealth,nodes={align=left,inner ysep=1pt},<-]
     % For censure_i
     \path (estimands.south) ++ (-3em,-0.5em) node[anchor=east,color=pink!200] (estititle){\text{Estimated distributions}};
     \draw [color=pink!200](estimands.south) -- ([xshift=-0.1ex,color=pink!200]estititle.east);
    % For censure_zeta
    \path (oracle.south) ++ (-3em,-0.5em) node[anchor=east,color=purple!67] (oracletitle){\text{Oracle distributions}};
     \draw [color=purple!87](oracle.south) -- ([xshift=-0.1ex,color=purple!87]oracletitle.east);
\end{tikzpicture}%
\end{equation}
When equality is achieved \emph{only} for the oracle distributions, the scoring rule is \emph{strictly proper}.
\end{definition}

\section{A STRICTLY PROPER SCORING RULE FOR COMPETING RISKS}

We prove that the negative log-likelihood, re-weighted by the censoring distribution (IPCW: Inverse Probabilities of Censoring Weights), is strictly proper.

\begin{definition}[Competitive Weights Negative LogLoss]
    We introduce the multiclass negative log-likelihood, re-weighted with the censoring distribution. The different classes represent the loss for all the cumulative incidence functions and the survival function.% 
    \begin{multline}
    \quad\forall \zeta, (\mathbf{x}, t, \delta) \sim \mathcal{D}, \\ \quad \mathrm{L}_{\zeta}((\hat{F}_1(\zeta| \mathbf{x}), ..., \hat{F}_K(\zeta| \mathbf{x}), \hat{S}(\zeta| \mathbf{x})), (t, \delta)) \defeq \\
    \frac{1}{n} \sum_{i=1}^n \left( \sum_{k=1}^{K} 
    \dfrac{
        \mathbb{1}_{t_i \leq \zeta, \delta_i = k} ~~\log\left(\hat{F}_k(\zeta|\mathbf{x}_i)\right)
        }
        {
        \tikzmarknode{censure_i}{\highlight{orange}
            {$G^*(t_i|\mathbf{x}_i) $}
            }} \right) \\ 
        +
        \dfrac{
        \mathbb{1}_{t_i > \zeta} ~~ \log\left(\hat{S}(\zeta|\mathbf{x}_i)\right)
        }
        {
        \tikzmarknode{censure_t}{\highlight{cyan}
            {$G^*(\zeta|\mathbf{x}_i) $}
        }}
    \label{eqn:full_loss}
\end{multline}
\begin{tikzpicture}[overlay,remember picture,>=stealth,nodes={align=left,inner ysep=1pt},<-]
     % For censure_i
     \path (censure_i.south) ++ (-3em, -2em) node[anchor=east,color=orange!67] (titlecensi){\text{\parbox{25ex}{Probability of remaining censor-free at $t_i$}}};
     \draw [color=orange!87](censure_i.west) -- ([xshift=-0.1ex,color=orange]titlecensi.north east);
    % For censure_zeta
    \path (censure_t.south) ++ (-3em,-1em) node[anchor=east,color=cyan!67] (censi){\text{\parbox{30ex}{Probability  of remaining censor-free at $\zeta$ \rlap{\small\quad (1 - probability of  censoring)}}}};
     \draw [color=cyan!87](censure_t.west) -- ([xshift=-0.1ex,color=cyan]censi.north east);
\end{tikzpicture}
\end{definition}
\vspace{.8em}

Eq.\ref{eqn:full_loss} is a standard log-loss (also known as cross-entropy), reweighted by appropriate sample weights —the inverse probabilities, or IPCW. Therefore, it can easily be added to most multiclass estimators.
\smallskip

\begin{restatable}{lemma}{usefullemma}
\label{lem:usefulllemma}%
Accounting for the time horizon $\zeta$, the expectation of the above scoring rule can be written as: $\quad\forall \zeta, (\mathbf{X}, T, \Delta) \sim \mathcal{D}, $
\begin{equation}
\small
\begin{aligned}
    \mathbb{E}_{T^*\!, C, \Delta |\mathbf{X} =\mathbf{x}}\!\left[\mathrm{L}_{\zeta}\!\left((\hat{F}_1(\zeta| \mathbf{x}), ..., \hat{F}_K(\zeta| \mathbf{x}), \hat{S}(\zeta| \mathbf{x})), (T, \Delta)\right)\right] \\
    = \sum_{k=1}^K \log\left(\hat{F}_k(\zeta|\mathbf{x})\right)
        F^*_k(\zeta |\mathbf{x}) 
     + \log\left(\hat{S}(\zeta|\mathbf{x})\right)
        S^*(\zeta| \mathbf{x})
        \label{eqn:loss}
\end{aligned}
\end{equation}
\end{restatable}%
\begin{proof}[Proof sketch]
    The weights allow us to transition from the observed distribution $T$ to the uncensored distribution $T^*$, which is crucial for demonstrating properness. The full proof can be found in Appendix \ref{prooflemma}.
\end{proof}%
\begin{restatable}[Properness of the scoring rule]{theorem}{bigthm}\label{thm:bigtheorem}
    Under the assumption that the weights are appropriately chosen, 
    $L_{\zeta}:  \mathbb{R}^{K+1} \times \mathcal{D} \rightarrow \mathbb{R}$ is a strictly proper scoring rule for the global CIF on a fixed time horizon $\zeta \in \mathbb{R}_+$.
\end{restatable}%
\begin{proof}[Proof sketch]
    Using the previous result, the properties of the negative log-likelihood, and Definition \ref{def:psrcr}, we conclude that the loss is strictly proper. Full proof in Appendix \ref{psrweights}.
\end{proof}

\section{\textsc{SurvivalBoost}: GRADIENT BOOSTING COMPETING RISKS}

\begin{figure*}[t]
%\scalebox{.8}
\centerline{\includegraphics[width=\textwidth]{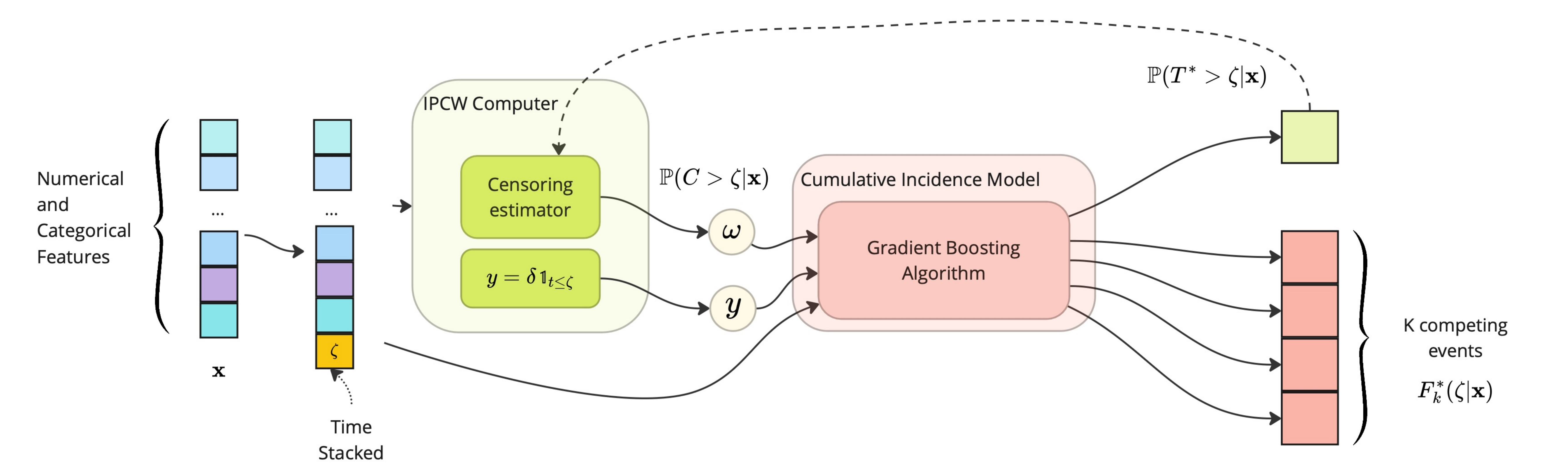}}%
\caption{\textbf{\textsc{SurvivalBoost} Algorithm with its Feedback Loop.} After providing input to the algorithm, a random time is assigned, and the corresponding weights and target are computed. After each iteration, the feedback loop updates the censoring probability, $G^\star$ as defined in eq.\ref{eqn:full_loss}.}
\label{architecture}
\end{figure*}

While eq.\ref{eqn:full_loss} can be used as a loss in any multiclass machine learning algorithm, we choose Gradient Boosting Trees due to their strong performance on tabular data \citep{grinsztajn_why_2022} and their compatibility with stochastic optimization. 
Gradient boosting methods approximate complex functions by combining weak learners (or base learners). At each iteration $m$, the algorithm focuses on the residuals of the loss function and builds a base learner $h_m$ that minimizes these residuals. For gradient boosting trees, the estimator typically takes the form $H_m(x) = H_{m-1}(x) + \nu h_m(x)$ where $\nu$ represents a chosen learning rate. For more on gradient boosting, refer to \citet{friedman}.

Most survival or competing risk loss functions cannot be used with tree-based models, as they require time derivates and thus smoothness. To address this, we introduce an algorithm called \textsc{SurvivalBoost}, which predicts all CIFs for each competing event as well as the global survival function. By predicting these jointly, we ensure that the stability of the probabilities is maintained, as the outputs of the classification models naturally sum to one. This ensures that $\mathbb{P}(T^* \leq \zeta|\mathbf{X}=\mathbf{x}) + \mathbb{P}(T^*>\zeta|\mathbf{X}=\mathbf{x}) = 1$, meaning the model's outputs are consistent and sum to one:
\begin{equation*}
\small 
    \sum_{k=1}^K \underbrace{\mathbb{P}(T^* \leq \zeta \cap \Delta^* = k|\mathbf{X}=\mathbf{x})}_{k^{th} \text{CIF}}+\underbrace{\mathbb{P}(T^*>\zeta|\mathbf{X}=\mathbf{x})}_{\text{Survival Probability}} = 1
\end{equation*}
Using the loss in eq.\ref{eqn:loss}, we can directly predict the CIF instead of predicting the hazard function (the derivative of the CIF), as is often done —for example, in DeepHit \citep{lee_deephit_2018} or SurvTRACE \citep{wang_survtrace_2022}. This approach allows us to drop the constant-hazard assumption present in \citep[][]{yanagisawa2023proper, kvamme_continuous_2019, wang_survtrace_2022, rindt_survival_2022}.

\begin{algorithm}[t]
       \caption{\textsc{SurvivalBoost} Algorithm - $m^{th}$ Iteration}
       \label{alg:incidence_one_iter}
    \begin{algorithmic}
       \STATE {\bfseries Input:} $\mathbf{x}, \delta, t$
       \FOR{$i=1$ {\bfseries to} $n_{samples}$}
       \STATE $\zeta_i \sim \mathcal{U}(0, t_{max})$ 
       \ENDFOR
       \STATE $\zeta \gets (\zeta_i)_{1 \leq i \leq n_{samples}}$
       \COMMENT{Sample a time horizon}
       \STATE $\tilde{\mathbf{x}} \gets (\mathbf{x}, \zeta)$ \COMMENT{Stacking the time to the features}
       \STATE $y, w \gets \text{ipcwComputer}(\mathbf{x}, \delta, t, \hat{G})$ \COMMENT{See Alg \ref{alg:ipcw_computer}}
       \STATE $L \gets \frac{1}{n} \sum_{i=1}^n \sum_{k=1}^{K} \left(
        \mathbb{1}_{y_i = k} ~y_i ~w_i~\log\left(\hat{F}_k(\zeta_i|\mathbf{x}_i)\right)
        \right) $
        
        \STATE $\qquad \qquad  + \mathbb{1}_{y_i = 0} ~y_i ~w_i \log\left(\hat{S}(\zeta_i|\mathbf{x}_i)\right)$
        \STATE $h_m(\mathbf{\tilde{x}}) \gets$ Train one iteration of Gradient Boost with $L$ as the loss
        \COMMENT{$h_m$ is the $m^{th}$ weak learner}
        \STATE $H_m(\mathbf{\tilde{x}}) \gets \nu h_{m}( \mathbf{\tilde{x}}) +  H_{m-1}( \mathbf{\tilde{x}})$ \COMMENT{$H_m$ is the estimator at the $m^{th}$ iteration, $\nu$ the learning rate}
       \STATE $(\hat{S}(\zeta |\mathbf{X}= \mathbf{x}), (\hat{F}_k(\zeta |\mathbf{X}= \mathbf{x})_{1\leq k \leq K}) \gets H_m(\mathbf{\tilde{x}})$ 
       \STATE $\hat{G} \gets$ Train one iteration of the Censoring-Feedback-Loop with $\hat{S}(\zeta |\mathbf{X}= \mathbf{x})$\COMMENT{See Alg \ref{alg:censo_one_iter} }
    \end{algorithmic}%
\end{algorithm}
Our algorithm utilizes two classifiers (here, gradient-boosted trees), one for censoring, trained on binary censored/non-censored labels (i.e., for time $\zeta$, $\mathbb{P}(C > \zeta| \mathbf{X}= \mathbf{x})$), and one for multiple events. Both the censoring and event models are adjusted using IPCW weights.
To compute these IPCW weights, we iterate the training using a feedback loop similar to boosting. First, we compute a survival censoring model. Then, using these probabilities, we initialize our \textsc{SurvivalBoost} algorithm. After several iterations, we apply a feedback loop to retrain the censoring model.\\ %ja{est-ce que c'est pertinent de dire une fois feedback loop et une fois optimization alternée pour que les gens des deux comprennent ? (vraie question, peut-être que c'est confusant)}
To capture complex temporal dependencies, we uniformly sample a time point for each observation and include it as an additional feature. Multiple time points can be sampled per iteration for each observation, generating a richer dataset where the targets vary based on the specific times sampled, thus providing a broader range of temporal information. This is enabled by our separable loss function.
An additional benefit is that we can predict the CIF at any time, unlike models optimized for a limited number of time points that require interpolation for other times. \\
\autoref{architecture} illustrates an iteration: we compute the weights $w_i$ and targets $y_i$ based on the sampled times for each individual (eq.\,\ref{eqn:full_loss}). Specifically, for censored samples, the corresponding weight is set to 0. A target $y_i \in \llbracket1, K\rrbracket$ indicates that the event of interest occurred before $\zeta$ and when $y_i = 0$, the individual has survived without experiencing any event. 
Algorithm \ref{alg:incidence_one_iter} gives pseudocode.

\section{COMPETING RISKS EXPERIMENTS}
\subsection{Evaluation Metrics For Competing Risks}

The evaluation is mainly performed on two metrics\footnote{We do not focus on the
C-index over time, as this metric is biased \citep{blanche_c-index_2019, rindt_survival_2022}}.

\paragraph{Evaluating the predicted probability} We extend the method proposed by \citet{graf_assessment_1999} and \citet{schoop_quantifying_2011}. The formula and a formal proof of the properness of the loss can be found in Appendix \ref{sec:evaluation_psr}. To avoid potential circularity with the loss function that we optimized, we apply this evaluation metric to the Brier Score rather than the log-loss. To evaluate the model across all time points, we sum the Brier Score over time, resulting in the \emph{Integrated Brier Score} (IBS).

\paragraph{Prediction accuracy in time}
In many applications, such as predictive maintenance or medicine, it is crucial to determine the first event a subject is likely to encounter. We use a validation metric to check, for each sample, whether the observed event is predicted as the most likely at given times, selected as before using quantiles. For example, for an individual who encounters event 2 at time $t$, the probability of surviving until $t$ should be the highest compared to the probabilities of encountering any other event. Additionally, the probability of encountering event 2 after $t$ should be the highest. To measure this, we adapt Multi-Class accuracy to different time points:

\begin{definition}[Prediction accuracy at time $\zeta$]
    For a fixed time horizon $\zeta$, and denoting survival to any event as index 0, define $\hat{y} = \argmax\limits_{k \in [0, K]} \hat{F}_{k}(\zeta | \mathbf{X}=  \mathbf{x})$, the most probable event at $\zeta$, and $y_\zeta = \mathbb{1}_{t\leq \zeta} \delta$. We remove censored individuals, and $n_{nc}$ represents the number of uncensored individuals at $\zeta$.
    \begin{equation}
        Acc(\zeta)= \frac{1}{n_{nc}} \sum_{i=1}^n \mathbb{1}_{\hat{y}_i = y_{i, \zeta}}~\mathbb{1}_{\overline{\delta_i = 0, t_i \leq \zeta}}
    \end{equation}
\end{definition}
\subsection{Experimental Settings}
\paragraph{Synthetic Dataset}
We design a synthetic dataset with linear relations between features and targets, as well as dependencies between the censoring distribution and the features (Appendix \ref{fig:distrib_synthe}). 
To create the synthetic dataset, for each sample, we draw $2 n_{events}$ parameters from a normal distribution. We then generate the event durations from a Weibull distribution based on those parameters. The observation is determined by the minimum duration and its associated event. The censoring event is computed using the same method.
\paragraph{SEER Dataset}
This dataset tracks 470,000 breast cancer patients for up to ten years,
with mortality due to various diseases as the outcomes. The censoring rate is
approximately $63\%$, and Figure \ref{fig:seer} shows the distribution of
events. Unlike \cite{lee_deephit_2018} (DeepHit) and
\cite{wang_survtrace_2022} (SurvTRACE), which focus on the
two most prevalent events and censor the others (undermining the competing risk framework), we consider three competing events, aggregating the remaining events into a third class. We also remove some
features following \citet{wang_survtrace_2022}.

\paragraph{Baselines}
We compare our approach with 7 other competing risks models from simpler models with \citet{aalen_survival_2008}'s global estimator and the \citet{fine_proportional_1999} linear model to more complex methods.  \\ 
We benchmark against tree-based approach - Random Survival Forests (RSF) \citep{ishwaran_random_2008} -, often criticized for its memory limitations. In our comparison, we also include several neural network-based models. This includes DeepHit \citep{lee_deephit_2018} which is trained with a ranking loss that combines the C-index with a negative log-likelihood, Deep Survival Machines (DSM) \citep{nagpal2021deep} which employ a graphical method for feature encoding and DeSurv \citep{danks_derivative-based_2022} solves Ordinal Differential Equations for continuous time predictions. Finally, we include a transformer-based model,  SurvTRACE \citep{wang_survtrace_2022} which is trained at three-time horizons (based on quantiles of observed event times) and at time 0. \\
To compute the Integrated Brier Score over time, other methods require linear interpolation of their trained times. For times beyond their trained intervals, we assume the incidence remains constant. In contrast, our method is trained on uniformly sampled time horizons, allowing for predictions at any time. \\
For fair model comparison, we use the same hyperparameter\hyp{}tuning time budget (grid in Appendix \ref{tab:gridsearch}).
\subsection{Results: Competing Risks}
\paragraph{Synthetic dataset}
\autoref{fig:tradeoff_competing_synthetic} illustrates the trade-off between statistical performance and training time for each model. Using the synthetic dataset, we are able to compute an oracle IBS. \textsc{SurvivalBoost} performs best in terms of IBS and is the fastest to train.  

% We also conduct different experiments on the synthetic dataset varying the number of training points (\autoref{ipsr}), the censoring rate (\autoref{censo}), and the number of features (\autoref{ipsrvstime}). More experiments on the synthetic data set can be found in the Appendix \ref{app:synthetic_results}.

\begin{figure}[t!]
\includegraphics[width=1.0\linewidth]{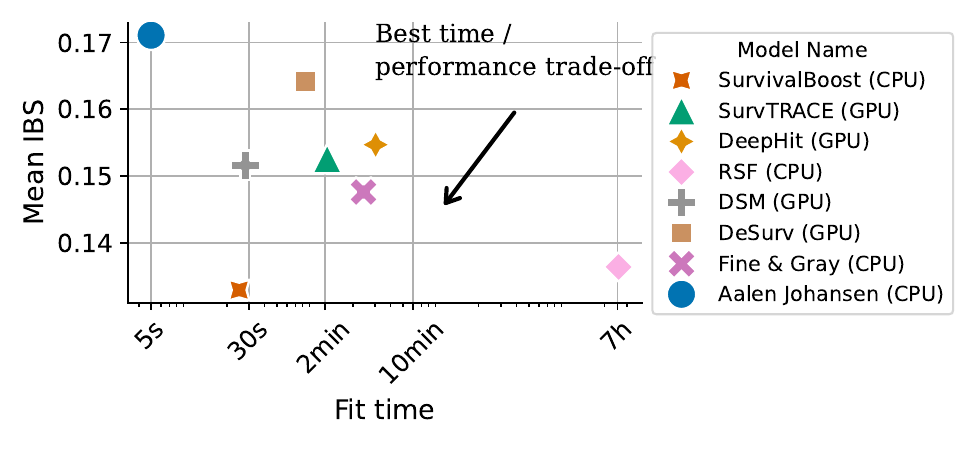}%
    \caption{\textbf{Prediction performance / training time trade-off for competing risk on the synthetic dataset.} Average IBS compared the fitting time for each model on 20k training data points, with a censoring rate of approximately 50\% and a dependant censoring across 6 features.}
\label{fig:tradeoff_competing_synthetic}%
\end{figure} %
\begin{figure}[t!]
\includegraphics[width=1.0\linewidth]{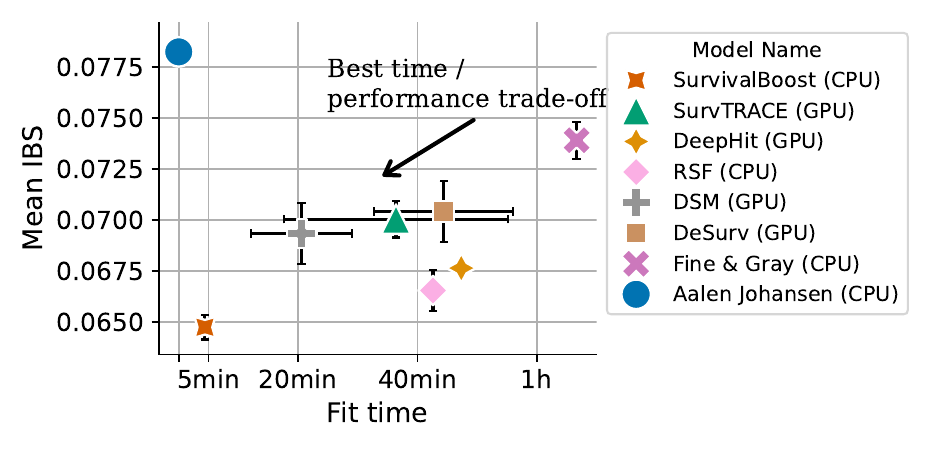}%
    \caption{\textbf{Prediction performance / training time trade-off for competing risks on SEER dataset}. Average IBS versus fitting time for each model, with a maximum of 330k training points, except for Fine \& Gray (50k) and RSF (100k). Table \ref{tab:ibs_event_seer} provides the IBS values for each event.}
    \label{fig:tradeoff_competing}
\end{figure} 
\vspace{-0.5em}
\paragraph{Results on SEER Dataset}
On the real-life dataset, we keep 30\% of the data for testing the models. 
\autoref{fig:tradeoff_competing} compares the models using the Integrated Brier Score (with Kaplan-Meier weights from
\cite{graf_assessment_1999} due to the absence of an oracle).
\textsc{SurvivalBoost} achieves both the best score and the shortest training time.
Random Survival Forest struggles with larger sample sizes (100k) and requires more than 50\,GB of RAM.
\textsc{SurvivalBoost} also maintains a significant lead with less training samples (Appendix \ref{app:seer_results}). 

\begin{figure}[t!]
	\includegraphics[width=\linewidth]{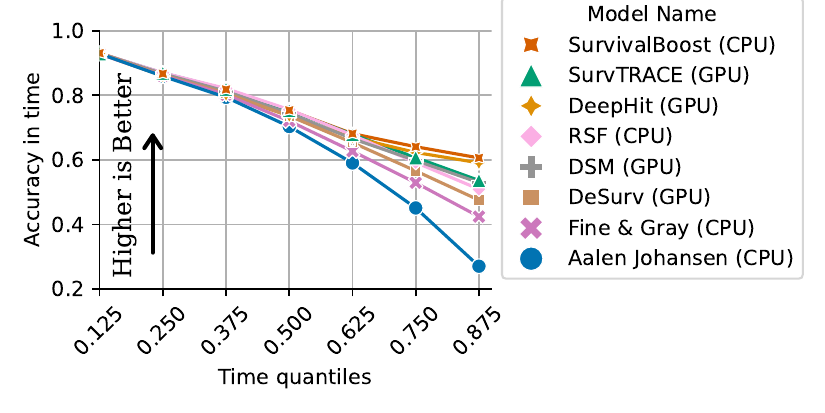}%
 \caption{\textbf{Prediction accuracy at time $\zeta$} Accuracy of the Argmax of the Cumulative Incidence Functions across different time quantiles on the SEER dataset.}
 \label{fig:acc}
\end{figure}

Event and time-specific C-indexes are presented in \autoref{tab:results_seer}, but they do not capture the models' ability to predict which event is more likely to occur at a given time horizon. This capability is measured by the accuracy in time, shown in \autoref{fig:acc}, where \textsc{SurvivalBoost} demonstrates the best performance. The advantage increases as time progresses, indicating that \textsc{SurvivalBoost} interpolates more effectively over time. 

\section{USAGE IN SURVIVAL ANALYSIS}
\subsection{Survival Experiments}
\paragraph{Real-life Datasets}
As our model can also handle survival analysis, we conducted experiments on three real-life survival datasets. 

\begin{description}[itemsep=1pt, parsep=1pt, topsep=0pt]
\item[METABRIC] \citep{metabric} The Molecular Taxonomy of Breast Cancer International Consortium dataset contains gene expression data with approximately 2,000 data points.
\item[SUPPORT] \citep{support} Study to Understand Prognoses Preferences Outcomes and Risks of Treatment dataset includes survival times for hospital patients, with more than 8,000 data points.
\item[KKBOX] The Churn Prediction Challenge 2017 hosted on \href{ https://www.kaggle.com/c/kkbox-churn-prediction-challenge}{Kaggle}, which features administrative censoring and 2.5M data points. We trained the models over 100k, 1M, and 2M data points to assess scalability (see Appendix, Fig. \ref{fig:cenlog-kkbox}).
\end{description}

\paragraph{Evaluation}
We use various metrics to evaluate models: the Integrated Brier Score (detailed in Appendix \ref{sec:evaluation_psr}) and another metric from \citet{yanagisawa2023proper}, called $S_{Cen-log-simple} \defeq S_{C-l-s}$ (detailed in Appendix \ref{sec:cen_log_simple}). 
Although this metric approximates the proper scoring metric from \citet{rindt_survival_2022}, it is not exactly proper (see Appendix \ref{sec:cen_log_simple}). It can be applied to any model as it does not require the density of the CIFs. 

\begin{table*}[t]
\caption{\textbf{Survival datasets}: Integrated Brier Score and $S_{C-l-s}$ (Lower is Better) depending on the size of each dataset. The \xmark \, indicates models that could not handle the data volume due to memory limitations.}
\small
\centering 
\begin{tabular}{l|rr|rr|rr}
\toprule
\hfill Dataset & \multicolumn{2}{c|}{\textbf{METABRIC (1k)}}&  \multicolumn{2}{c|}{\textbf{SUPPORT (8k)}} & \multicolumn{2}{c}{\textbf{KKBOX (1M)}} \\
\toprule
Model& IBS& $S_{C-l-s}$ & IBS & $S_{C-l-s}$ & IBS & $S_{C-l-s}$ \\
\midrule
Kaplan-Meier & .185±.010 & 2.039±.218 & .208±.004 & 1.617±.268 & .213±.001 & 1.723±.002\\
DeepHit &.171±.009&2.039±.001 & .207±.004 & 1.771±.000 & .147±.001 & 1.609±.002\\
PCHazard&.169±.011&1.980±.086 &.187±.004&1.673±.004& \underline{.107±.002}& 1.286±.002\\
\citeauthor{han2021inverse} &.196±.004&2.665±.036&.253±.002&3.223±.005& \xmark & \xmark\\
DQS8&.172±.018&2.200±.000 &.202±.004&2.764±.12& .119±.001 & 3.791±.027\\
SuMo net&.170±.010&2.197±.000&.194±.006&1.818±.000& \xmark & \xmark\\
SurvTRACE&.172±.006&1.987±.088&.188±.004&1.606±.003& .111±.002& \underline{1.270±.008}\\
RSF &\textbf{.165±.025}&\textbf{1.937±.227}& \underline{.182±.004}&1.942±.023&\xmark & \xmark \\
GBS & .169±.011 & \underline{1.974±.404} & .187±.004 & \underline{1.575±.001} &.157±.001 &1.511±.001\\
\bfseries \textsc{SurvivalBoost}&\underline{.168±.019}&2.027±.159 &\textbf{.181±.005}& \textbf{1.569±0.341} & \textbf{.105±.001}& \textbf{1.183±.029}\\
\bottomrule
\end{tabular}
\label{tab:survival}%
\end{table*}
\vspace{-0.5em}
\paragraph{Baselines}
We benchmark our method against the most performant competing risks and SOTA survival models. This includes neural networks such as DeepHit \citep{lee_deephit_2018} and PCHazard \citep{kvamme_continuous_2019}, as well as those trained with proper survival analysis scoring rules, such as SumoNet \citep{rindt_survival_2022}, and DQS \citep{yanagisawa2023proper}. We also evaluate transformer methods with SurvTRACE \citep{wang_survtrace_2022},
survival games \citep{han2021inverse}, and tree-based methods with Random Survival Forests (RSF) \citep{ishwaran_random_2008} and Gradient Boosting Survival Analysis (GBS) - from \href{https://scikit-survival.readthedocs.io/en/v0.23.0/api/generated/sksurv.ensemble.GradientBoostingSurvivalAnalysis.html}{Scikit-survival} \citep{sksurv}.

\subsection{Results: Survival Analysis}

\begin{figure}[b!]
    \includegraphics[width= \columnwidth]{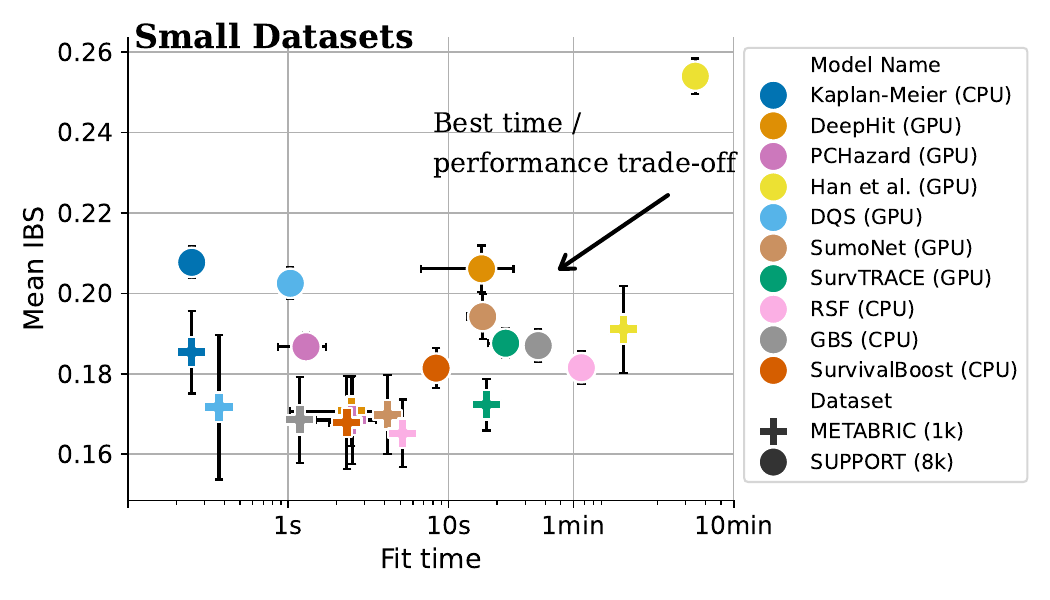}

    \includegraphics[width= \columnwidth]{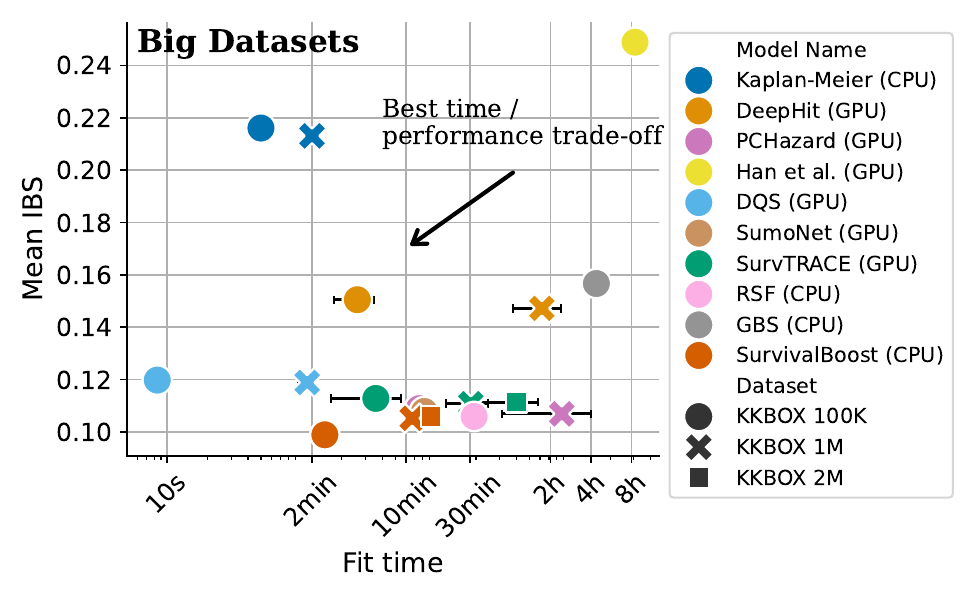}%
    \caption{\textbf{Prediction performance / training time trade-off in survival analysis} IBS (Integrated Brier score) function of fit time for each model on real-life datasets. For the big datasets, some algorithms exceeded computing resources.}
    \label{fig:tradeoff-ibs-survival}\label{fig:tradeoffkkbox}
\end{figure}

\autoref{fig:tradeoff-ibs-survival} shows the trade-off between training time and performance in terms of IBS, where \textsc{SurvivalBoost} excels, being the top model in statistical performance and one of the fastest on the datasets with enough data (SUPPORT and KKBOX) while being one of the best models for smaller datasets (METABRIC). Appendix \ref{tradeoffyana} provides a similar figure for the $S_{Cen-log-simple}$ metric, where \textsc{SurvivalBoost} achieves an excellent trade-off rivaled only by SumoNet, which has comparable performance on the $S_{Cen-log-simple}$ loss. Varying the sample size from 100k to 2M on the KKBOX dataset confirms that \textsc{SurvivalBoost} and DQS are faster (taking less than 1 minute on 100k data points), while \citeauthor{han2021inverse}, SumoNet, and RSF are slower for larger sample size. They exhibit super-linear time complexity, making them impractical for large datasets; for more than 100k data points they exceed memory limitations (See Appendix \ref{infra}). \\
\autoref{tab:survival} report evaluation metrics, including $S_{Cen-log-simple}$ which is not what \textsc{SurvivalBoost} directly optimizes. Across datasets, \textsc{SurvivalBoost} achieves the best results in terms of IBS and is tied with SumoNet for $S_{Cen-log-simple}$ (also for C-index, Appendix \ref{saall}). It is worth noting that SumoNet uses $S_{Cen-log-simple}$ as its training loss. However, this metric is not guaranteed to be a proper scoring rule, meaning it does not necessarily ensure accurate recovery of the true risks. For KKBOX, we only show the results for 1M data points.

Beyond proper scores, we investigate calibration, MAE, MSE, and the AUC adapted for survival analysis (Appendix \ref{tab:metabric-othermetrics}, \ref{tab:support-othermetrics}, \ref{tab:kkbox-othermetrics}). We assess the calibration using four tests, including  distribution calibration ($D_c$) \citep{haider_effective_2018} and One-time calibration ($\textsc{one}_c$) \citep{hosmer_comparison_1997}. Kaplan-Meier, \textsc{SurvivalBoost}, and RSF are the most calibrated models (Appendix \ref{tab:calibration-metrics}).

\section*{DISCUSSION AND CONCLUSION}
\paragraph{Code reproducibility and data}
The code will be made available on GitHub as a library.

\paragraph{Combination of tree-based architecture and loss function makes the difference}
Our work shares similarities with the equations in \cite{han2021inverse}, which also uses IPCW \cite[introduced by ][]{robins_estimation_1994}, though for survival and not competing risks.
Their learning strategy targets an equilibrium, showing that it recovers the oracle distribution in survival analysis settings. Meanwhile, our optimization uses a loss on all classes to compute the censoring distribution, while the other part optimizes only for the survival distribution. This last part departs from the schema in \cite{han2021inverse}. Despite similarities, the two approaches behave markedly different our empirical study.\\
Building upon trees-based model is probably important to this difference and to the success of
\textsc{SurvivalBoost}. Yet, comparing to GBS and RSF show that trees in themselves do not suffice. Our loss is crucial for scalability (as it is separable) and to facilitate fitting trees, as it avoids the need for time derivatives. It avoids issues that plague many competing risks methods.
The excellent empirical results, superior performance with less computational resources, come from combining the loss function with the tree-based approach results in a very stable algorithm. This double gain is especially valuable as health datasets continue to grow in size.
%\paragraph{Social impact} Our contribution is not directly applied and has no immediate social impact, but we hope that it will improve medical applications where survival analysis is central.
% \paragraph{Limitations and further work}
% Further work should consider removing the assumption of non-informative censoring. This assumption is very common in the literature, though some recent work has relaxed it in survival settings \citep{foomani2023copulabased,zhang2023deep}.

\paragraph{Acknowledgments}
JA, JA, and GV acknowledge funding from the ERC grand INTERCEPT-T2D.

\paragraph{Limitations and further work} Further work should consider removing the assumption of non-informative censoring \ref{info_censoring}. This assumption is very common in the literature, though some recent work has relaxed it in survival settings \cite{foomani2023copulabased, zhang2023deep}.

\paragraph{Conclusion}

For competing risks, which generalizes survival analysis to classify the type of outcome, we first propose and prove a strictly proper scoring rule. This reweighted log loss can easily be used in machine learning models: it is separable by observation, making it suitable for stochastic solvers, it does not require time derivatives (unlike most survival models) and it can be applied to non-differentiable models. We integrate it into gradient-boosting trees, resulting in an algorithm called \textsc{SurvivalBoost}. By using time as a feature and incorporating a feedback loop to better estimate censoring probabilities, \textsc{SurvivalBoost} outperforms state-of-the-art methods on both synthetic and real-life datasets, for both competing risks (classification on time-censored data) and standard survival analysis (time-to-event regression with right censoring). It also trains faster on large datasets. As a loss function, it allows survival analysis or competing risks modeling to be easily extended to a wide range of models— from scalable linear models to deep learning architectures, including fine-tuning foundation models— replacing clinical standards like \citeauthor{fine_proportional_1999} that do not scale.

\bibliography{papier2}

\bibliographystyle{plainnat}

\newpage

%%%%%%%%%%%%%%%%%%%%%%%%%%%%%%%%%%%%%%%%%%%%%%%%%%%%%%%%%%%%%%%%%%%%%%%%%%%%%%%
%%%%%%%%%%%%%%%%%%%%%%%%%%%%%%%%%%%%%%%%%%%%%%%%%%%%%%%%%%%%%%%%%%%%%%%%%%%%%%%
% APPENDIX
%%%%%%%%%%%%%%%%%%%%%%%%%%%%%%%%%%%%%%%%%%%%%%%%%%%%%%%%%%%%%%%%%%%%%%%%%%%%%%%
%%%%%%%%%%%%%%%%%%%%%%%%%%%%%%%%%%%%%%%%%%%%%%%%%%%%%%%%%%%%%%%%%%%%%%%%%%%%%%%
\newpage
\appendix

 \xdef\presupfigures{\arabic{figure}}% save the current figure number
  \xdef\presuptables{\arabic{table}}% save the current figure number
\renewcommand\thefigure{S\fpeval{\arabic{figure}-\presupfigures}} 
\renewcommand\thetable{S\fpeval{\arabic{table}-\presuptables}} 

\onecolumn

\section{Definitions}
\subsection{Notations}

Below, we detail the notations used throughout the main manuscript, as well as in the proofs and derivations.

The following conventions apply to all symbols:
\begin{itemize}
    \item $.^*$: Oracle
    \item $\hat{.}$: Estimation
\end{itemize}

The different variables that we use are:
\begin{table}[h!]
\begin{tabularx}{\linewidth}{r r X}
\toprule
Maths Symbol & Domain & Description\\
\midrule 
    $\zeta$ & $\mathbb{R}_+$& Time horizon \\
\midrule
    $K$ & $\mathbb{N}^*$ & number of competing events (events of interest)\\ 
    $\mathbf{X}$ & $\mathcal{X}$& random variable representing an individual \\
    $T^*_k$ &$\mathbb{R}_+$& random variable of the time-to-event for event $k$\\
    $C$ &$\mathbb{R}_+$& random variable of the time-to-censoring\\
    $T^*$ & $\mathbb{R}_+$ & $\min (T^*_1, T^*_2, ..., T^*_K)$\\
    $T$ & $\mathbb{R}_+$ & $\min (T, C)$ \\
    $\Delta^*$ &$[1, K]$&$\argmin\limits_{k \in [1, K]} (T^*_k)$\\
    $\Delta$ &$[0, K]$&$\argmin (C, T^*_1, T^*_2, ..., T^*_K)$\\
    \midrule
    S & $\mathcal{S}$ & Survival function\\
    F & $\mathcal{F}$ & Cumulative Incidence Function \\
    G & $\mathcal{G}$ & Censor function\\
    \midrule
    $n$ &$\mathbb{N}^*$& number of individuals in our observation \\
    $i$ &$[1, n]$& one observation\\
    $\mathbf{x}_i$ &$\mathcal{X}^n$& individuals observed \\
    $t_i$ &$\mathbb{R}_+^n$& time-to-event/censoring observed\\
    $\delta_i$ & $[0, K]$& event observed, 0 indicates censoring \\
\bottomrule
\end{tabularx}
\caption{Notations used}
\label{tab:history}
\end{table}

\subsection{Reporting conventions}

In the tables, the best results are highlighted in bold, and the second-best results are underlined.

% \section{Bilevel-optimization}
% The optimization problem can be reformulated as a bilevel optimization problem. In this framework, with a fixed estimator for the censoring distribution, the optimal distributions for the various CIFs are determined by minimizing the loss function, which is reweighted based on the estimated censoring distribution. This corresponds to the upper-level objective. Meanwhile, the censoring distribution itself is estimated as part of the lower-level objective.

% \begin{multline}
%     \forall \zeta, (\mathbf{x}, t, \delta) \sim \mathcal{D},  \\        
% \end{multline}

% \begin{equation}
% \begin{aligned}
% \argmin\limits_{(F_1, .., F_K, S)}  \quad &
%     \frac{1}{n} \sum_{i=1}^n \sum_{k=1}^{K} \left(
%     \dfrac{
%         \mathbb{1}_{t_i \leq \zeta, \delta_i = k} ~~\log\left(F_k(\zeta|\mathbf{x}_i)\right)
%         }
%         {
%         G(t_i|\mathbf{x}_i) 
%             } \right)
%         +
%         \dfrac{
%         \mathbb{1}_{t_i > \zeta} ~~ \log\left(S(\zeta|\mathbf{x}_i)\right)
%         }
%         {
%         G(\zeta|\mathbf{x}_i)
%         } \\
% \textrm{s.t.} \quad & G \in \argmin\limits_{\hat{G}} 
%     \frac{1}{n} \sum_{i=1}^n \left(
%     \dfrac{
%         \mathbb{1}_{t_i \leq \zeta, \delta_i = 0} ~~\log\left(1 - \hat{G}(\zeta|\mathbf{x}_i)\right)
%         }
%         {
%         S(t_i|\mathbf{x}_i) 
%             } \right)
%         +
%         \dfrac{
%         \mathbb{1}_{t_i > \zeta} ~~ \log\left(\hat{G}(\zeta|\mathbf{x}_i)\right)
%         }
%         {
%         S(\zeta|\mathbf{x}_i)
%         }\\
% \end{aligned}
% \end{equation}

\section{Theory on our proper scoring rule: proofs and derivations}
In this appendix, we give the proofs and derivations concerning the proper scoring rule that we have introduced.

\usefullemma*
\begin{proof}[Proof the of Lemma \ref{lem:usefulllemma} on the expectation of the Reweighted NLL]\label{prooflemma}

\begin{multline}
    \forall \zeta, \forall k \in \llbracket1, K\rrbracket, (\mathbf{x}, t, \delta) \sim \mathcal{D}, \\
    \mathrm{L}_{\zeta}\left((\hat{F}_1(\zeta| \mathbf{x}), ..., \hat{F}_K(\zeta| \mathbf{x}), \hat{S}(\zeta| \mathbf{x})), (t, \delta)\right) \defeq \frac{1}{n} \sum_{i=1}^n \left(\sum_{k=1}^{K}
    \underbrace{
    \dfrac{
        \mathbb{1}_{t_i \leq \zeta, \delta_i = k} ~~\log\left(\hat{F}_k(\zeta|\mathbf{x}_i)\right)
        }
        {
        \tikzmarknode{censure_i_eval}{\highlight{orange}
            {$G^*(t_i|\mathbf{x}_i) $}
            }}}_{\defeq \Psi_{k, \zeta}(\hat{F}_k(\zeta| \mathbf{x}), ( t, \delta))} \right)
        +
        \underbrace{\dfrac{
        \mathbb{1}_{t_i > \zeta} ~~ \log\left(\hat{S}(\zeta|\mathbf{x}_i)\right)
        }
        {
        \tikzmarknode{censure_t_eval}{\highlight{cyan}
            {$G^*(\zeta|\mathbf{x}_i) $}
        }}}_{\defeq \Lambda_{k, \zeta}(\hat{S}(\zeta| \mathbf{x}), ( t, \delta))}
\end{multline}

For the next computations, we recall the definition of the different variables.

\paragraph{Computation of the expectation:} 
First: 
\begin{align}
    \mathbb{E}_{T^*, C, \Delta |\mathbf{X} =\mathbf{x}}\left[\Psi_{k, \zeta}(\hat{F}_k(\zeta| \mathbf{x}), ( T, \Delta))\right] &= 
    \mathbb{E}_{T^*, C, \Delta |\mathbf{X} =\mathbf{x}}\left[ 
                    \mathbb{1}_{T \leq \zeta} \mathbb{1}_{\Delta = k}  
                    \dfrac{\log\left(\hat{F}_k(\zeta|\mathbf{x})\right)}{G^*(T|\mathbf{x})}  
                    \right] \\
    &=\log\left(\hat{F}_k(\zeta|\mathbf{x})\right) ~\mathbb{E}_{T^*, C, \Delta |\mathbf{X} =\mathbf{x}}\left[
                    \dfrac{
                    \mathbb{1}_{\min(T^*, C) \leq \zeta}
                    \mathbb{1}_{\Delta = k} 
                    }
                    {
                    G^*(T|\mathbf{x})
                    }
                    \right] \\
    &=\log\left(\hat{F}_k(\zeta|\mathbf{x})\right) ~\mathbb{E}_{T^*, C, \Delta |\mathbf{X} =\mathbf{x}}\left[
                    \dfrac{
                    (\mathbb{1}_{T^*\leq \zeta}\mathbb{1}_{T^*\leq C} + \mathbb{1}_{C\leq \zeta}\mathbb{1}_{C\leq T^*}
                    )
                    \mathbb{1}_{\Delta = k} 
                    }
                    {
                    G^*(T|\mathbf{x})
                    }
                    \right] \\
    &=\log\left(\hat{F}_k(\zeta|\mathbf{x})\right) ~\mathbb{E}_{T^*, C, \Delta |\mathbf{X} =\mathbf{x}}\left[
                    \dfrac{
                    \mathbb{1}_{T^*\leq \zeta}\mathbb{1}_{T^*\leq C}
                    \mathbb{1}_{\Delta = k} 
                    }
                    {
                    G^*(T|\mathbf{x})
                    }
                    +
                    \underbrace{\dfrac{
                    \mathbb{1}_{C\leq \zeta}\mathbb{1}_{C\leq T^*}
                    \mathbb{1}_{\Delta = k} 
                    }
                    {
                    G^*(T|\mathbf{x})
                    }}_{= 0 \text{ because $k\neq 0$}}
                    \right] \\
&=\log\left(\hat{F}_k(\zeta|\mathbf{x})\right) ~\mathbb{E}_{T^*, C, \Delta |\mathbf{X} =\mathbf{x}}\left[
                    \dfrac{
                    \mathbb{1}_{T^* \leq \zeta} \mathbb{1}_{T^* \leq C} 
                    \mathbb{1}_{\Delta = k} 
                    }
                    {
                    G^*(T|\mathbf{x})
                    }
                    \right] \\
    &= \log\left(\hat{F}_k(\zeta|\mathbf{x})\right)
        \mathbb{P}(T^*\leq \zeta, \Delta = k | \mathbf{X}= \mathbf{x}) 
\end{align}
The last equality can be expanded as follows:
\begin{align}
    \mathbb{E}_{T^*, C, \Delta |\mathbf{X} =\mathbf{x}}\left[
                    \dfrac{
                    \mathbb{1}_{T^* \leq \zeta} \mathbb{1}_{T^* \leq C} 
                    \mathbb{1}_{\Delta = k} 
                    }
                    {
                    G^*(T|\mathbf{x})
                    }
                    \right]
                    &= \int_0^\infty \int_0^\infty(\mathbb{1}_{\min(t, c) = t} + \underbrace{\mathbb{1}_{\min(t,c)= c}}_{= 0   \text{~because $k\neq 0$}}) \dfrac{
                    \mathbb{1}_{t \leq \zeta} \mathbb{1}_{t \leq c} 
                    }
                    {
                    G^*(t|\mathbf{x})
                    }
                    f_{T^*, C, \Delta}(t, c, k |\mathbf{x}) dt \,  dc  \\
                    &\text{T is a composition of $T^*$ and $C$ } \nonumber \\
                    &=\int_0^\infty \int_0^\infty \dfrac{
                    \mathbb{1}_{t \leq \zeta} \mathbb{1}_{t \leq c} 
                    }
                    {
                    G^*(t|\mathbf{x})
                    }
                    f_{T^*, C, \Delta}(t, c, k |\mathbf{x}) dt \, dc \\
                    &= \int_0^\infty \int_0^\infty \dfrac{
                    \mathbb{1}_{t \leq \zeta} \mathbb{1}_{t \leq c} 
                    }
                    {
                    G^*(t|\mathbf{x})
                    }
                    f_{T^*, \Delta}(t, k|\mathbf{x}) f_{C}(c|\mathbf{x}) dt \,  dc \\
                    &\text{Because  } T^* \indep C |\mathbf{X} \nonumber \\
                    &=  \int_0^\infty \dfrac{
                    \mathbb{1}_{t \leq \zeta}
                    }
                    {
                    G^*(t|\mathbf{x})
                    }
                    f_{T^*, \Delta}(t, k|\mathbf{x}) \left(\int_0^\infty  \mathbb{1}_{t \leq c} f_{C}(c|\mathbf{x}) dc \right) dt  \\ 
                    &=  \int_0^\infty \dfrac{
                    \mathbb{1}_{t \leq \zeta}
                    }
                    {
                    G^*(t|\mathbf{x})
                    }
                    f_{T^*, \Delta}(t, k|\mathbf{x}) \left(G^*(t|\mathbf{x}) \right) dt  \\ 
                    &\text{with the definition of $G^*$} \nonumber\\
                    &=  \int_0^\infty 
                    \mathbb{1}_{t \leq \zeta}
                    f_{T^*, \Delta}(t, k|\mathbf{x}) dt \\
                    &= \mathbb{P}(T^*\leq \zeta, \Delta = k | \mathbf{X}= \mathbf{x})
\end{align}

And: 
\begin{align}
    \mathbb{E}_{T, \Delta |\mathbf{X} =\mathbf{x}}\left[\Lambda_{k, \zeta}(\hat{S}(\zeta| \mathbf{X}= \mathbf{x}), (T, \Delta))|\mathbf{X}=\mathbf{x}\right] &= 
    \mathbb{E}_{T, \Delta |\mathbf{X} =\mathbf{x}}\left[
                    \mathbb{1}_{T>\zeta}
                    \dfrac{\log\left(\hat{S}(\zeta| \mathbf{X}= \mathbf{x})\right)}{G^*(\zeta|\mathbf{x})}  
                    \right] \\
        &= \log\left(\hat{S}(\zeta| \mathbf{X}= \mathbf{x})\right) \mathbb{E}_{T, \Delta |\mathbf{X} =\mathbf{x}}\left[ 
                    \dfrac{
                    \mathbb{1}_{\min(T^*, C)>\zeta}
                    }
                    {G^*(\zeta|\mathbf{x})} 
                    \right]\\
        &= \log\left(\hat{S}(\zeta| \mathbf{X}= \mathbf{x})\right) \mathbb{E}_{T, \Delta |\mathbf{X} =\mathbf{x}}\left[ 
                    \dfrac{
                    \mathbb{1}_{T^*>\zeta}\mathbb{1}_{C>\zeta}
                    }
                    {G^*(\zeta|\mathbf{x})} 
                    \right]\\
        &= \log\left(\hat{S}(\zeta| \mathbf{X}= \mathbf{x})\right) \mathbb{E}_{T, \Delta |\mathbf{X} =\mathbf{x}}\left[ 
                    \dfrac{
                    \mathbb{1}_{C>\zeta}
                    }
                    {G^*(\zeta|\mathbf{x})} 
                    \right]\mathbb{E}_{T, \Delta |\mathbf{X} =\mathbf{x}}\left[ 
                    \mathbb{1}_{T^*>\zeta}
                    \right]\\
        &\text{Because  } T^* \indep C |\mathbf{X} \nonumber\\
        &= \log\left(\hat{S}(\zeta| \mathbf{X}= \mathbf{x})\right) 
                    \dfrac{
                    \mathbb{E}_{T, \Delta |\mathbf{X} =\mathbf{x}}\left[ \mathbb{1}_{C>\zeta}\right]
                    }
                    {G^*(\zeta|\mathbf{x})} 
                    \mathbb{E}_{T, \Delta |\mathbf{X} =\mathbf{x}}\left[ 
                    \mathbb{1}_{T^*>\zeta}
                    \right]\\
        &\text{Because } G^*(\zeta|\mathbf{x}) \text{ does not depend of } T \text{ and } \Delta \nonumber\\
        &= \log\left(\hat{S}(\zeta| \mathbf{X}= \mathbf{x})\right) \mathbb{P}(T^*>\zeta| \mathbf{X} = \mathbf{x})\\
\end{align}

By summing all of the terms, we obtain: 

\begin{multline}
    \mathbb{E}_{T^*, C, \Delta |\mathbf{X} =\mathbf{x}}\left[\mathrm{L}_{\zeta}\left((\hat{F}_1(\zeta| \mathbf{x}), ..., \hat{F}_K(\zeta| \mathbf{x}), \hat{S}(\zeta| \mathbf{x})), (T, \Delta)\right)\right] \\
    = \sum_{k=1}^K \log\left(\hat{F}_k(\zeta|\mathbf{x})\right)
        \mathbb{P}(T^*\leq \zeta, \Delta = k | \mathbf{X}=\mathbf{x}) \\ 
     + \log\left(\hat{S}(\zeta| \mathbf{X}= \mathbf{x})\right) \mathbb{P}(T^*>\zeta| \mathbf{X} = \mathbf{x}) 
\end{multline}
\begin{align}
    &=\sum_{k=1}^K \log\left(\hat{F}_k(\zeta|\mathbf{x})\right)
        F^*_k(\zeta |\mathrm{x})  
     + \log\left(\hat{S}(\zeta|\mathbf{x})\right)
        S^*(\zeta| \mathbf{x})
\end{align}

Finally:
\begin{multline}
    \mathbb{E}_{T^*, C, \Delta |\mathbf{X} =\mathbf{x}}\left[\mathrm{L}_{\zeta}\left((\hat{F}_1(\zeta| \mathbf{x}), ..., \hat{F}_K(\zeta| \mathbf{x}), \hat{S}(\zeta| \mathbf{x})), (T, \Delta)\right)\right] \\
    =\sum_{k=1}^K \log\left(\hat{F}_k(\zeta|\mathbf{x})\right)
        F^*_k(\zeta |\mathrm{x})  
     + \log\left(\hat{S}(\zeta|\mathbf{x})\right)
        S^*(\zeta| \mathbf{x})
\end{multline}

\end{proof}

\begin{proof}[Proof of the Theorem \ref{thm:bigtheorem}]\label{psrweights}
\bigthm*
To be more explicit, we can define a new random variable $Y$: 
\begin{definition}
     $$\forall \zeta, ~ Y_{k, \zeta} \defeq T^* \leq \zeta \cap \Delta = k$$
     And:
     $$\forall \zeta, ~ Y_{0, \zeta} \defeq T^* > \zeta$$
\end{definition}  

Thus, the previously mentioned quantities of interest can be rewritten as functions of these variables:
\begin{equation}
    F^*_k(\zeta |\mathbf{x}) = \mathbb{P}(T^*\leq \zeta, \Delta = k| \mathbf{X} = \mathbf{x}) = \mathbb{P}(Y_{k, \zeta}=1| \mathbf{X} = \mathbf{x})
\end{equation}
\begin{equation}
    S^*(\zeta |\mathbf{x}) = \mathbb{P}(T^*> \zeta| \mathbf{X} = \mathbf{x}) = \mathbb{P}(Y_{0, \zeta}=1| \mathbf{X} = \mathbf{x})
\end{equation}

$\hat{F}_k(\zeta|\mathbf{x})$ represents the estimated probability that $Y_{k, \zeta}=1$, so we rewrite it as $\hat{p}_{k, \zeta} \defeq \hat{F}_k(\zeta|\mathbf{x})$. \\
Therefore: 
\begin{flalign}
    \mathbb{E}_{T^*, C,\Delta |\mathbf{X} =\mathbf{x}}\left[\mathrm{L}_{k, \zeta}(\hat{F}_k(\zeta |\mathbf{x}), ( T, \Delta))\right] &=\mathbb{E}_{T, \Delta |\mathbf{X} =\mathbf{x}}[\mathrm{L}_{\zeta}(\hat{p}_{\zeta}, ( T, \Delta))] \\
    &= \sum_{k=0}^K \log\left(\hat{p}_{k, \zeta}\right)
        \mathbb{P}(Y_{k, \zeta}=1| \mathbf{X} = \mathbf{x})  \\
    &\text{Using Lemma \ref{lem:usefulllemma}} \nonumber
\end{flalign}
Thus, we obtain the following optimization problem:
\begin{equation}
\begin{aligned}
\max_{\hat{p}} \quad &  \sum_{k=0}^K \log\left(\hat{p}_{k, \zeta}\right)
        \mathbb{P}(Y_{k, \zeta}=1| \mathbf{X} = \mathbf{x}) \\
\textrm{s.t.} \quad & \sum_{k=0}^K \hat{p}_k = 1\\
  &\hat{p}_k\geq0    \\
\end{aligned}
\end{equation}
The problem can be reformulated as a convex optimization problem due to the concavity of the logarithm: 
\begin{equation}
\begin{aligned}
\min_{\hat{p}} \quad &  - \sum_{k=0}^K \log\left(\hat{p}_{k, \zeta}\right)
        \mathbb{P}(Y_{k, \zeta}=1| \mathbf{X} = \mathbf{x})\\
\textrm{s.t.} \quad & \sum_{k=0}^K \hat{p}_k = 1\\
  &\hat{p}_k\geq0    \\
\end{aligned}
\end{equation}

We apply the Karush-Kuhn-Tucker conditions since the constraints are qualified (as they are linear). These conditions imply that if $p$ is a local minimum of the problem, there exits $\lambda \in \mathbb{R}$ and $\mu \in \mathbb{R}_+^{K+1}$ such that: 

\begin{align}
    &\nabla \left(- \sum_{k=0}^K \log\left(\hat{p}_{k, \zeta}\right)  \mathbb{P}(Y_{k, \zeta}=1| \mathbf{X} = \mathbf{x})\right) - \mu ^\top \mathbf{1}_K + \lambda = 0 \\
    &\forall k, \mu_k \hat{p}_{k, \zeta} = 0
\end{align}

If $\exists k, \hat{p}_{k,\zeta} = 0 \implies - \sum_{k=0}^K \log\left(\hat{p}_{k, \zeta}\right)
        \mathbb{P}(Y_{k, \zeta}=1| \mathbf{X} = \mathbf{x}) = +\infty$. \\
Hence, equation (36) implies that $\forall k, \mu_k = 0$. \\

Now,
\begin{align}
    & \forall k, \frac{\partial \left(- \sum_{k=0}^K \log\left(\hat{p}_{k, \zeta}\right)  \mathbb{P}(Y_{k, \zeta}=1| \mathbf{X} = \mathbf{x})\right)}{\partial \hat{p}_{k, \zeta}} = -\frac{\mathbb{P}(Y_{k, \zeta}=1| \mathbf{X} = \mathbf{x})}{\hat{p}_{k, \zeta}} \\
    &\text{(37) can be rewritten as:} \nonumber \\
    &\forall k, \quad - \frac{\mathbb{P}(Y_{k, \zeta}=1| \mathbf{X} = \mathbf{x})}{\hat{p}_{k, \zeta}} + \lambda = 0 \\
    &\Longrightarrow \forall k, - \mathbb{P}(Y_{k, \zeta}=1| \mathbf{X} = \mathbf{x}) + \lambda\hat{p}_{k, \zeta} = 0 \\
    &\text{By summing over } k, \nonumber\\
    &\Longrightarrow - \underbrace{\sum_{k=0}^K  \mathbb{P}(Y_{k, \zeta}=1| \mathbf{X} = \mathbf{x})}_{= 1} + \lambda \underbrace{\sum_{k=0}^K\hat{p}_{k, \zeta}}_{=1} = 0 \\
    &\Longrightarrow \lambda = 1\\
    &\Longrightarrow \forall k, \hat{p}_{\zeta, k} =  \mathbb{P}(Y_{k, \zeta}=1| \mathbf{X} = \mathbf{x})
\end{align}

Any local minimum must satisfy the KKT conditions. Therefore, if $p$ is a local minimum, it is a solution to equations (34) and (42). Consequently, as shown above, the only possible solution must be equal to the oracle distribution. Indeed, the loss is strictly proper. 

\end{proof}

\section{Study of the proper scoring rule used for evaluation}%
\label{sec:evaluation_psr}

As mentioned earlier, the most commonly used metric in the competing risks setting, the C-index over time, is known to be biased \citep{blanche_c-index_2019, rindt_survival_2022}. To address this significant issue in evaluation strategies, we propose two alternative evaluation metrics: one based on a reweighted proper scoring rule, which can be applied to any proper binary scoring rule, and another based on accuracy over time, which measures the observed event against the most likely predicted event.

\subsection{PSR for evaluation}
The PSR introduced in the main paper as the loss function of our algorithm serves as a global loss across all predictions. The following loss is adapted to focus on a specific event $k$, allowing us to evaluate our estimates for that event. In the paper, we focus on the IBS, though one could alternatively use a logarithmic loss because of its properness.  
\paragraph{Proper scoring rule for the $k^{th}$ competing event}
In our setting, we denote $L_{k, \zeta}$ as a scoring rule for the $k^{th}$ CIF at a time horizon $\zeta$. 
\begin{definition}[\emph{PSR for the $k^{th}$ cause-specific event}]
The scoring rule $L_{k, \zeta}$ for the $k^{th}$ CIF at time $\zeta$ for an observation $(\mathbf{X}, T, \Delta)$ is proper if and only if:
 \begin{flalign}
   \forall \zeta, (\mathbf{X}, T, \Delta )\sim \mathcal{D}, ~~
    \mathbb{E}_{T^*, C, \Delta  | \mathbf{X}= \mathbf{x}}[L_{k, \zeta}(\hat{F}_k(\zeta| \mathbf{x}), (T, \Delta))]  \leq \mathbb{E}_{T^*, C, \Delta  | \mathbf{X}= \mathbf{x}}[L_{k, \zeta}(F^*_k(\zeta| \mathbf{x}), (T, \Delta))]
\end{flalign}
\end{definition}

\subsubsection{A proper scoring rule for competing risks}
To evaluate our model, we used the following proper scoring rule, which is appropriate for each event. This proper scoring rule allows us to assess the error for each specific event and the global error across all CIFs. 

In the following, we prove that any given (strictly) proper scoring rule that can be used in the multiclass setting (\emph{e.g.} the Brier score or negative log-likelihood) leads to a (strictly) proper scoring in competing risks settings by re-weighting the observations. \\
Indeed, for any (strictly) proper scoring rule $\ell: \mathbb{R} \times \{0,1\} \rightarrow \mathbb{R}$, we can construct a cause-specific scoring rule function $L_{k, \zeta}:  \mathbb{R} \times \mathcal{D} \rightarrow \mathbb{R}$, which is also a (strictly) proper scoring rule for the $k^{th}$ cause-specific event at the fixed time horizon $\zeta \in \mathbb{R}_+$. It follows that $L_{\zeta}$ is (strictly) proper. 
\begin{definition}[\emph{PSR with re-weighting}]
We define $L_{k,\zeta}$, considering the observations $(\mathbf{x}, t, \delta)$ for an event $k$, as the following scoring rule for the $k^{th}$ CIF:
\begin{multline}
    \forall \zeta, \forall k \in \llbracket1, K\rrbracket, \ell: \mathbb{R} \times \{0,1\}\rightarrow \mathbb{R}, (\mathbf{x}, t, \delta) \sim \mathcal{D} \\
    \mathrm{L}_{k, \zeta}(\hat{F}_k(\zeta| \mathbf{x}), (t, \delta)) \defeq \frac{1}{n} \sum_{i=1}^n 
    \dfrac{
        \mathbb{1}_{t_i \leq \zeta, \delta_i = k} ~~\ell\left(\hat{F}_k(\zeta|\mathbf{x}_i), 1\right)
        }
        {
        \tikzmarknode{censure_ie1}{\highlight{orange}
            {$G^*(t_i|\mathbf{x}_i) $}
            }} \\
        +
        \dfrac{
        \mathbb{1}_{t_i > \zeta} ~~ \ell\left(\hat{F}_k(\zeta|\mathbf{x}_i), 0\right)
        }
        {
        \tikzmarknode{censure_te}{\highlight{cyan}
            {$G^*(\zeta|\mathbf{x}_i) $}
        }} \\
    + 
        \dfrac{
        \mathbb{1}_{t_i \leq \zeta, \delta_i \neq 0, \delta_i \neq k} ~~ \ell\left(  \hat{F}_k(\zeta|\mathbf{x}_i), 0\right)
        }
        {
        \tikzmarknode{censure_ie}{\highlight{orange}
            {$G^*(t_i|\mathbf{x}_i) $}
            }} 
\end{multline}%
\begin{tikzpicture}[overlay,remember picture,>=stealth,nodes={align=left,inner ysep=1pt},<-]
     % For censure_i
     \path (censure_ie.south) ++ (-10em, 1em) node[anchor=east,color=orange!67] (titlecensi){\text{Probability of remaining at $t_i$}};
     \draw [color=orange!87](censure_ie.west) -- ([xshift=-0.1ex,color=orange]titlecensi.east);
    % For censure_zeta
    \path (censure_te.south) ++ (-8em,1em) node[anchor=east,color=cyan!67] (censi){\text{\parbox{25ex}{Probability  of remaining at $\zeta$ \\ \small (1 - probability of  censoring)}}};
     \draw [color=cyan!87](censure_te.west) -- ([xshift=-0.1ex,color=cyan]censi.east);
\end{tikzpicture}
The weights correspond to the Inverse Probability of Censoring Weighting (IPCW), which is used to re-calibrate the observed population to align with the uncensored oracle population \cite{robins_estimation_1994}. This PSR is
an extension of \citet{graf_assessment_1999} and
\citet{schoop_quantifying_2011} when $\ell$ is the Brier Score.
\end{definition} 

\begin{lemma}\label{expectation}
    Considering a proper scoring rule $\ell:\mathbb{R} \times \{0,1\}$, at time horizon $\zeta$ and for any cause-specific risk $k$, the expectation of the scoring rule can be expressed as: 
    \begin{multline}
        \forall \zeta, \forall k \in \llbracket1, K\rrbracket, \ell: \mathbb{R} \times \{0,1\}\rightarrow \mathbb{R}, (\mathbf{X}, T, \Delta) \sim \mathcal{D}, \\
        \mathbb{E}_{T^*, C, \Delta |\mathbf{X} =\mathbf{x}}\left[\mathrm{L}_{k, \zeta}\left(\hat{F}_k(\zeta|\mathbf{x}), (T, \Delta)\right)\right] = 
    \ell\left(\hat{F}_k(\zeta|\mathbf{x}), 1\right)
        F^*_k(\zeta|\mathbf{x}) 
     + \ell\left(\hat{F}_k(\zeta|\mathbf{x}), 0\right)
        \left(1 - F^*_k(\zeta|\mathbf{x})\right)
    \end{multline}
\end{lemma} 

\begin{proof}
The computations are essentially the same as in the previous section.
    \begin{multline}
    \forall \zeta, \forall k \in \llbracket1, K\rrbracket, \ell: \mathbb{R} \times \{0,1\}\rightarrow \mathbb{R}, (\mathbf{x}, t, \delta) \sim \mathcal{D} \\
    \mathrm{L}_{k, \zeta}(\hat{F}_k(\zeta| \mathbf{x}), (t, \delta)) \defeq \frac{1}{n} \sum_{i=1}^n 
    \underbrace{
    \dfrac{
        \mathbb{1}_{t_i \leq \zeta, \delta_i = k} ~~\ell\left(\hat{F}_k(\zeta|\mathbf{x}_i), 1\right)
        }
        {G^*(t_i|\mathbf{x}_i)}
    }_{\defeq \Psi_{k, \zeta}(\hat{F}_k(\zeta| \mathbf{x}), ( t, \delta))} \\
        +
    \underbrace{
        \dfrac{
        \mathbb{1}_{t_i > \zeta} ~~ \ell\left(\hat{F}_k(\zeta|\mathbf{x}_i), 0\right)
        }
        {G^*(\zeta|\mathbf{x}_i)} 
    }_{\defeq \Lambda_{k, \zeta}(\hat{F}_k(\zeta| \mathbf{x}), ( t, \delta))} \\
    + 
    \underbrace{
        \dfrac{
        \mathbb{1}_{t_i \leq \zeta, \delta_i \neq 0, \delta_i \neq k} ~~ \ell\left(  \hat{F}_k(\zeta|\mathbf{x}_i), 0\right)
        }
        {G^*(t_i|\mathbf{x}_i)} 
    }_{\defeq \Phi_{k, \zeta}(\hat{F}_k(\zeta| \mathbf{x}), ( t, \delta))} 
\end{multline}
\begin{align}
    \mathbb{E}_{T^*, C, \Delta |\mathbf{X} =\mathbf{x}}\left[\Psi_{k, \zeta}(\hat{F}_k(\zeta| \mathbf{x}), ( T, \Delta))|\mathbf{X}=\mathbf{x}\right] &= 
    \mathbb{E}_{T^*, C, \Delta |\mathbf{X} =\mathbf{x}}\left[ 
                    \mathbb{1}_{T \leq \zeta} \mathbb{1}_{\Delta = k}  
                    \dfrac{\ell\left(\hat{F}_k(\zeta|\mathbf{x}), 1\right)}{G^*(T|\mathbf{x})}  
                    \right] \\
    &=\ell\left(\hat{F}_k(\zeta|\mathbf{x}), 1\right) ~\mathbb{E}_{T^*, C, \Delta |\mathbf{X} =\mathbf{x}}\left[
                    \dfrac{
                    \mathbb{1}_{T^* \leq \zeta} \mathbb{1}_{T^* \leq C} 
                    \mathbb{1}_{\Delta = k} 
                    }
                    {
                    G^*(T|\mathbf{x})
                    }
                    \right] \\
    &= \ell\left(\hat{F}_k(\zeta|\mathbf{x}), 1\right)
        \mathbb{P}(T^*\leq \zeta, \Delta = k | \mathrm{X}= \mathrm{x}) \\
\end{align}

\begin{align}
    \mathbb{E}_{T^*, C, \Delta |\mathbf{X} =\mathbf{x}}\left[\Phi_{k, \zeta}\left(\hat{F}_k(\zeta| \mathbf{x}), ( T, \Delta)\right)\right] &= 
    \mathbb{E}_{T^*, C, \Delta |\mathbf{X} =\mathbf{x}}\left[ 
                    \mathbb{1}_{T\leq \zeta, \Delta \neq 0, \Delta \neq k}
                    \dfrac{\ell\left(\hat{F}_k(\zeta|\mathbf{x}), 0\right)}{G^*(T|\mathbf{x})}  
                    \right] \\
    &=\ell\left(\hat{F}_k(\zeta|\mathbf{x}), 0\right) ~\mathbb{E}_{T^*, C, \Delta |\mathbf{X} =\mathbf{x}}\left[ 
                    \dfrac{
                    \mathbb{1}_{T^* \leq \zeta} \mathbb{1}_{T^* \leq C} 
                    \mathbb{1}_{\Delta \neq k} 
                    }
                    {
                    G^*(T|\mathbf{x})
                    }
                    \right] \\
    &= \ell\left(\hat{F}_k(\zeta|\mathbf{x}), 0\right)
        \mathbb{P}(T^*\leq \zeta, \Delta \neq k | \mathrm{X}= \mathrm{x})\\
\end{align}

\begin{align}
    \mathbb{E}_{T^*, C, \Delta |\mathbf{X} =\mathbf{x}}\left[\Lambda_{k, \zeta}(\hat{F}_k(\zeta, \mathbf{x}), (T, \Delta))|\mathbf{X}=\mathbf{x}\right] &= 
    \mathbb{E}_{T^*, C, \Delta |\mathbf{X} =\mathbf{x}}\left[
                    \mathbb{1}_{T>\zeta}
                    \dfrac{\ell\left(1 - \hat{F}_k(\zeta|\mathbf{x}), 0\right)}{G^*(\zeta|\mathbf{x})}  
                    \right] \\
        &= \ell\left(\hat{F}_k(\zeta|\mathbf{x}), 0\right) \mathbb{E}_{T^*, C, \Delta |\mathbf{X} =\mathbf{x}}\left[] 
                    \dfrac{
                    \mathbb{1}_{T^*>\zeta}
                    \mathbb{1}_{C>\zeta}
                    }
                    {\mathbb{P}(C>\zeta|\mathbf{x})} 
                    \right]\\
        &= \ell\left(\hat{F}_k(\zeta|\mathbf{x}), 0\right) \mathbb{P}(T^*>\zeta| \mathbf{X} = \mathbf{x})\\
\end{align}

By summing all of the terms, we obtain: 
\begin{equation}
\begin{multlined}
    \mathbb{E}_{T^*, C, \Delta |\mathbf{X} =\mathbf{x}}\left[\mathrm{L}_{k, \zeta}\left(\hat{F}_k(\zeta|\mathbf{x}), ( T, \Delta)\right)\right] = 
    \ell\left(\hat{F}_k(\zeta|\mathbf{x}), 1\right)
        \mathbb{P}(T^*\leq \zeta, \Delta = k) \\ 
     + \ell\left(\hat{F}_k(\zeta|\mathbf{x}), 0\right)
        \left( \mathbb{P}(T^*\leq \zeta, \Delta \neq k | \mathrm{X}= \mathrm{x})+ \mathbb{P}(T^*>\zeta| \mathbf{X} = \mathbf{x})\right)
\end{multlined}
\end{equation}

Meanwhile, 
\begin{flalign}
    \mathbb{P}(\overline{T^* \leq \zeta \cap \Delta = k}) &= \mathbb{P}(T^* > \zeta \cup \Delta \neq k) \\
    &= \mathbb{P}(T^* > \zeta) + \mathbb{P}(\Delta \neq k) - \mathbb{P}(T^* > \zeta \cap \Delta \neq k) \\ 
    &= \mathbb{P}(T^* > \zeta) + \mathbb{P}(\Delta \neq k \cap T^* > \zeta) + \mathbb{P}(\Delta \neq k \cap T^* \leq \zeta) - \mathbb{P}(T^* > \zeta \cap \Delta \neq k) \\ 
    &= \mathbb{P}(T^* > \zeta) + \mathbb{P}(\Delta \neq k \cap T^* \leq \zeta)
\end{flalign}
Therefore, we obtain:
\begin{equation}
    \mathbb{E}_{T^*, C, \Delta |\mathbf{X} =\mathbf{x}}\left[\mathrm{L}_{k, \zeta}\left(\hat{F}_k(\zeta|\mathbf{x}), ( T, \Delta)\right)\right]= 
    \ell\left(\hat{F}_k(\zeta|\mathbf{x}), 1\right)
        F^*_k(\zeta|\mathbf{x})
     + \ell\left(\hat{F}_k(\zeta|\mathbf{x}), 0\right)
        \left(1 - F^*_k(\zeta|\mathbf{x})\right)
\end{equation}
\end{proof}

\begin{proposition}\label{psr}
    If $\ell: \mathbb{R} \times \{0,1\} \rightarrow \mathbb{R}$ is a chosen (strictly) proper scoring rule, then $L_{k, \zeta}:  \mathbb{R} \times \mathcal{D} \rightarrow \mathbb{R}$ is also a (strictly) proper scoring rule for the $k^{th}$ cause-specific event at the fixed time horizon $\zeta \in \mathbb{R}_+$.
\end{proposition}

\begin{proof}
    \begin{equation}
\begin{multlined}
    \mathbb{E}_{T^*, C, \Delta |\mathbf{X} =\mathbf{x}}\left[\mathrm{L}_{k, \zeta}\left(\hat{F}_k(\zeta|\mathbf{x}), ( T, \Delta)\right)\right] = 
    \ell\left(\hat{F}_k(\zeta|\mathbf{x}), 1\right)
        \mathbb{P}(T^*\leq \zeta, \Delta = k | \mathbf{X}= \mathbf{x}) \\ 
     + \ell\left(\hat{F}_k(\zeta|\mathbf{x}), 0\right)
        \left( \mathbb{P}(T^*\leq \zeta, \Delta \neq k | \mathbf{X}= \mathbf{x})+ \mathbb{P}(T^*>\zeta| \mathbf{X} = \mathbf{x})\right)
\end{multlined}
\end{equation}

To be more explicit, we define a new random variable $Y$: 
\begin{definition}
     $$\forall \zeta, ~ Y_{k, \zeta} \defeq T^* \leq \zeta \cap \Delta = k$$
\end{definition} 

\begin{equation}
    F^*_k(\zeta |\mathbf{x}) = \mathbb{P}(T^*\leq \zeta, \Delta = k| \mathbf{X} = \mathbf{x}) = \mathbb{P}(Y_{k, \zeta}=1| \mathbf{X} = \mathbf{x})
\end{equation}
$\hat{F}_k(\zeta|\mathbf{x})$ represents the estimated probability that $Y_{k, \zeta}=1$, allowing us to rewrite it as: $\hat{p}_{k, \zeta} \defeq \hat{F}_k(\zeta|\mathbf{x}) \approx \mathbb{P}(Y_{k, \zeta}=1| \mathbf{X} = \mathbf{x})$
Therefore: 
\begin{flalign}
    \mathbb{E}_{T^*, C, \Delta |\mathbf{X} =\mathbf{x}}\left[\mathrm{L}_{k, \zeta}(\hat{F}_k(\zeta |\mathbf{x}), (T^*, C, \Delta))\right] &=\mathbb{E}_{T, \Delta |\mathbf{X} =\mathbf{x}}[\mathrm{L}_{k,\zeta}(\hat{p}_{k, \zeta}, ( T, \Delta))] \\
    &=\ell\left(\hat{p}_{k, \zeta}, 0\right)
        \mathbb{P}(Y_{k, \zeta}=0| \mathbf{X} = \mathbf{x})+
    \ell\left(\hat{p}_{k, \zeta}, 1\right)
         \mathbb{P}(Y_{k, \zeta}=1| \mathbf{X} = \mathbf{x}) \\
    &= \mathbb{E}_{Y_{k, \zeta}}[\ell(\hat{p}_{k, \zeta}, Y_{k, \zeta})| \mathbf{X} = \mathbf{x}] \\
    &\leq \mathbb{E}_{Y_{k, \zeta}}[\ell(p_{k, \zeta}, Y_{k, \zeta})| \mathbf{X} = \mathbf{x}]\\
    &\leq \mathbb{E}_{T^*, C, \Delta |\mathbf{X} =\mathbf{x}}[\mathrm{L}_{k, \zeta}(\mathbb{P}(Y_{k, \zeta}=1), ( T, \Delta))] \\
    &\leq \mathbb{E}[\mathrm{L}_{k, \zeta}(F^*_k(\zeta |\mathrm{x}), ( T, \Delta))]
\end{flalign}

The last inequality holds because $l$ is a proper scoring rule.  Similarly, the same computation leads to a strictly proper scoring rule if $l$ is strictly proper. \\

Thus, we conclude that $\forall \zeta, \forall k \in \llbracket1, K\rrbracket, ~ \mathrm{L}_{k, \zeta}(\hat{F}_k(\zeta |\mathrm{x}), ( T, \Delta))$ is a proper scoring rule of $F_k^*(\zeta | \mathbf{x})$. \\
% Because we have shown that $F^* = \sum_{k=1}^K F^*_k$,  $\mathrm{L} = \sum_{k=1}^K \mathrm{L}_k$ is a proper scoring rule.
\end{proof}

\begin{theorem}\label{gpsr}
    If $\ell: \mathbb{R} \times \{0,1\} \rightarrow \mathbb{R}$, a chosen (strictly) proper scoring rule, then $L_{\zeta}:  \mathbb{R} \times \mathcal{D} \rightarrow \mathbb{R}$ is a (strictly) proper scoring rule for the global CIF at a fixed time horizon $\zeta \in \mathbb{R}_+$.
\end{theorem}

\begin{proof}
    This follows straightforwardly from the proposition and the lemma above.
\end{proof}

\paragraph{Corollary: Proper global scoring rule to compare competing risk models}
The defined scoring rule $\sum_{k=1}^K\mathrm{L}_{k, \zeta}$ is proper on any arbitrarily chosen time horizon $\zeta$. To compare
different models, a global measure is necessary, such as summing
over time, as introduced by \citet{graf_assessment_1999}. Here, we extend the Integrated Brier Score to other (strictly) proper scoring rules $l$ and prove that the Integrated Loss (IL) is also a (strictly) proper scoring rule. \\
By considering:
$$Z \sim \mathcal{U}(0, t_{max})$$ with $t_{max}$ being the maximum time horizon for prediction.

\begin{definition}[\emph{Integrated global PSR}]
    With $\ell: \mathbb{R} \times \{0,1\} \rightarrow \mathbb{R}$, a chosen scoring rule, the cause-specific scoring rule function $L_{k, \zeta}:  \mathbb{R} \times \mathcal{D} \rightarrow \mathbb{R}$  defined as above, we define the $\mathrm{IL}$ as
    \begin{flalign}
        \mathrm{IL}(\hat{F}_1(.| \mathbf{x}), ..., \hat{F}_K(.| \mathbf{x}), ( T, \Delta))
        &\defeq \mathbb{E}_{Z}\left[\sum_{k=1}^K \mathrm{L}_{k, Z}(\hat{F}_{k}(Z| \mathbf{x}), ( T, \Delta))| \mathbf{X} = \mathbf{x}\right] \\
        &= \sum_{k=1}^K \underbrace{
        \mathbb{E}_{Z}\left[\mathrm{L}_{k, Z}(\hat{F}_k(Z| \mathbf{x}), ( T, \Delta))| \mathbf{X} = \mathbf{x}\right]
        }_{\defeq \mathrm{IL}_k(\hat{F}_{k}(.| \mathbf{x}), ( T, \Delta))}
    \end{flalign}
\end{definition}

\begin{corollary}\label{ipsr}
    With $\ell: \mathbb{R} \times \{0,1\} \rightarrow \mathbb{R}$, a chosen (strictly) proper scoring rule, the cause-specific loss function $L_{k, \zeta}:  \mathbb{R} \times \mathcal{D} \rightarrow \mathbb{R}$ defined above $\mathrm{IL}$ is a (strictly) proper scoring rule.
\end{corollary}

\begin{proof}
We have already proven that $L_{k, \zeta}:  \mathbb{R} \times \mathcal{D} \rightarrow \mathbb{R}$ is a (strictly) proper scoring rule. Given the monotonicity and positivity of the expectation, the result follows immediately. 

\begin{align}
    \mathbb{E}_{T^*, C, \Delta | \mathbf{X}= \mathbf{x}, Z=\zeta} \left[\mathrm{IL}_k(\hat{F}_k(\zeta| \mathbf{x})), ( T, \Delta) \right] &= \mathbb{E}_{T^*, C, \Delta | \mathbf{X}= \mathbf{x}, Z=\zeta} \left[\mathrm{L}_k(\hat{F}_k(\zeta| \mathbf{x}), ( T, \Delta))\right] \\
    &\leq \mathbb{E}_{T^*, C, \Delta | \mathbf{X}= \mathbf{x}, Z=\zeta} \left[\mathrm{L}_k(F^*_k(\zeta| \mathbf{x}), ( T, \Delta))\right] \\
    &\leq \mathbb{E}_{T^*, C, \Delta | \mathbf{X}= \mathbf{x}, Z=\zeta} \left[\mathrm{IL}_k(F^*_k(\zeta| \mathbf{x}), ( T, \Delta))\right]
\end{align}
And since the expectation is non-decreasing, we have:
\begin{align}
    \mathbb{E}_{T^*, C, \Delta} \left[\mathrm{IL}_k(\hat{F}_k(Z| \mathbf{x}), (T, \Delta))| \mathbf{X}= \mathbf{x}\right] 
    &\leq \mathbb{E}_{T^*, C, \Delta} \left[\mathrm{IL}_k(F^*_k(Z| \mathbf{x}), (T, \Delta))| \mathbf{X}= \mathbf{x}\right]
\end{align}
This allows us to consider the Integrated Loss (IL) as a global proper scoring rule for comparing different competing risks models. 
\end{proof}

\section{Examples}
\subsection{Brier Score}
When we define $l(y, \hat{y}) \defeq (y - \hat{y})^2$, we obtain the censoring-adjusted Brier score for the $k^{th}$ competing event, as defined in equation 14 of \cite{kretowska_tree-based_2018}:

\begin{definition}
\begin{multline}
    \forall \zeta, \forall k \in [1, K], \\
    \mathrm{BS}_k(\hat{F}_k(\zeta, \mathbf{x}), \delta, t, \zeta, \mathbf{x}) \defeq \frac{1}{n} \sum_{i=1}^n 
    \dfrac{
        \mathbb{1}_{t_i \leq \zeta, \delta_i = k}\left(1 - \hat{F}_k(\zeta|\mathbf{x}_i)\right)^2
        }
        {G^*(t_i|\mathbf{x}_i)}
        +
        \dfrac{
        \mathbb{1}_{t_i > \zeta} \left(\hat{F}_k(\zeta|\mathbf{x}_i)\right)^2
        }
        {G^*(\zeta|\mathbf{x}_i)} 
     \\
    + 
        \dfrac{
        \mathbb{1}_{t_i \leq \zeta, \delta_i \neq 0, \delta_i \neq k} \left(  \hat{F}_k(\zeta|\mathbf{x}_i)\right)^2
        }
        {G^*(t_i|\mathbf{x}_i)}
\end{multline}
\end{definition}

\subsection{Binary cross entropy loss} 
As explained by \citet{benedetti_scoring_2010}, the log loss captures uncertainty better than the mean squared error. Therefore, one could also evaluate survival analysis and competing risks models using the following loss.

{\footnotesize
\begin{multline}
    \forall k \in [1, K], \\
    \mathrm{l}_k(\hat{F}_k(\zeta, \mathbf{x}), \delta, t, \zeta)\defeq \frac{1}{n} \sum_{i=1}^n 
    \dfrac{
        \mathbb{1}_{t_i \leq \zeta, \delta_i = k}\log\left( \hat{F}_k(\zeta|\mathbf{x}_i)\right)
        }
        {G^*(t_i|\mathbf{x}_i)}
        + 
        \dfrac{
        \mathbb{1}_{t_i \leq \zeta, \delta_i \neq 0, \delta_i \neq k} \log\left(1 -  \hat{F}_k(\zeta|\mathbf{x}_i)\right)
        }
        {G^*(t_i|\mathbf{x}_i)} \\ 
        + 
        \dfrac{
        \mathbb{1}_{t_i > \zeta} \log\left(1 - \hat{F}_k(\zeta|\mathbf{x}_i)\right)
        }
        {G^*(\zeta|\mathbf{x}_i)}
\end{multline}
}

\section{The \citet{yanagisawa2023proper} scoring rule for survival \label{sec:cen_log_simple}} 

\citet{yanagisawa2023proper} introduce a metric called $S_{Cen-log-simple}$, which is an approximation of the proper scoring metric in \citet{rindt_survival_2022}.
The metric in \citet{rindt_survival_2022} requires the hazard function, which is the time derivative of the cumulative incidence function. This derivative can only be computed by differentiable models, implying an implicit assumption on almost-everywhere smooth time dependence. To avoid the need for the hazard function, \citet{yanagisawa2023proper} approximate it as piecewise affine. They demonstrate that under the assumption that the ``node time points'' —the edges of the affine segments— match an actual piecewise-affine breakdown of the CIF, the resulting approximation is proper. They argue that with enough node time points, this metric serves as a good approximation of a proper scoring rule.

$S_{Cen-log-simple}$ is defined as: 
\begin{multline}
    S_{Cen-log-simple}( \hat{F} , (t, \delta ); \{\zeta_i\}^B_{i=0}) \defeq \\
    -\delta \sum^{B-1}_{i=0} \mathbb{1}_{\zeta_i < t \leq \zeta_{i+1}} \log(\hat{F}(\zeta_{i+1}) - \hat{F}(\zeta_i)) \\
    - (1 - \delta) \sum_{i=0}^{B- 1} \mathbb{1}_{\zeta_i < t \leq \zeta_{i+1}} \log(1 - \hat{F}(\zeta_{i+1}))
\end{multline}
where $B$ is the number of node time points\footnote{We use $B=32$, as in the experiments in \citet{yanagisawa2023proper}}, and $\{\zeta_i\}^B_{i=0}$ are the node times points, evenly spaced between $0$ and $t_{max}$, dividing the time space into $B$ equal intervals.

\section{Pseudo-code}
\begin{algorithm}[h!]
       \caption{IPCW Computer}
       \label{alg:ipcw_computer}
    \begin{algorithmic}
       \STATE {\bfseries Input:} $\mathbf{x}, \delta, t, \hat{G}$
       \STATE $y \gets \delta ~ \mathbb{1}_{t \leq \zeta}$ \COMMENT{Computing the target}
       \IF[The observation is not censored]{$t > \zeta$}
            \STATE $w \gets \frac{1}{\hat{G}(\zeta|\mathbf{x})}$
        \ELSIF{$t \leq \zeta$ and $\delta \neq 0$}
            \STATE $w \gets \frac{1}{\hat{G}(t|\mathbf{x})}$
        \ELSE
            \STATE $w \gets 0$
        \ENDIF
    % \STATE $w \gets (w_i)_{1 \leq i \leq n_{samples}}$
    \RETURN $y, w$
    \end{algorithmic}
\end{algorithm}

\begin{algorithm}[h!]
       \caption{Censoring Feedback Loop - One Iteration}
       \label{alg:censo_one_iter}
    \begin{algorithmic}
       \STATE {\bfseries Input:} $\mathbf{x}, \delta, t, \hat{S}$
       \FOR{$i=1$ {\bfseries to} $n_{samples}$}
       \STATE $\zeta_i \sim \mathcal{U}(0, t_{max})$ 
       \ENDFOR
       \STATE $\zeta \gets (\zeta_i)_{1 \leq i \leq n_{samples}}$
       \STATE $\Tilde{\mathbf{x}} \gets (\mathbf{x}, \zeta)$
       \STATE $\delta \gets \mathbb{1}_{y = 0} $ \COMMENT{Changing the target (focusing on the censoring distribution)}
       \STATE $y, w \gets \text{ipcwcomputer}(\mathbf{x}, \delta, t, \hat{S})$ 
       \COMMENT{See Alg \ref{alg:ipcw_computer}}
       \STATE $L \gets \frac{1}{n} \sum_{i=1}^n \left(
        y_i ~w_i~\log\left(1 - \hat{G}_k(\zeta_i|\mathbf{x}_i)\right)
        \right)
        +
        (1 - y_i) ~w_i \log\left(\hat{G}(\zeta_i|\mathbf{x}_i)\right)$
        \STATE $\tilde{h}_m(\mathbf{\tilde{x}}) \gets$ Train one iteration of Gradient Boost with $L$ as the loss \COMMENT{$\tilde{h}_m$ is the $m^{th}$ weak learner}
        \STATE $\tilde{H}_m(\zeta | \mathbf{x}) \gets \tilde{h}_{m}(\zeta | \mathbf{x}) + \nu \tilde{H}_{m-1}(\zeta | \mathbf{x})$ \COMMENT{$\tilde{H}_m$ is the $m^{th}$ estimator}
       \STATE $((1 - \hat{G})(\zeta |\mathbf{X}= \mathbf{x}), \hat{G}(\zeta |\mathbf{X}= \mathbf{x})) \gets \tilde{H}_m(\mathbf{\tilde{x}})$ 
       \RETURN $\hat{G}(\zeta |\mathbf{X}= \mathbf{x})$
    \end{algorithmic}
\end{algorithm}

\section{Additional results for competing risk experiments}

\subsection{Results in the survival analysis setting}
\subsubsection{KKBOX}
Here, we present the results of the experiments conducted on the KKBOX dataset (Figures \ref{fig:cenlog-kkbox} and \ref{fig:tradeoff-ibs-survival}).  We highlight the trade-offs observed to assess the scalability of the models. Specifically, the models were trained on KKBOX using subsamples of 100k, 1M, and 2M training data points. However, due to computational constraints, it was not possible to run some experiments with 1M or 2M data points.

\begin{figure}[h!]
    \centering
    \includegraphics[width=0.7\linewidth]{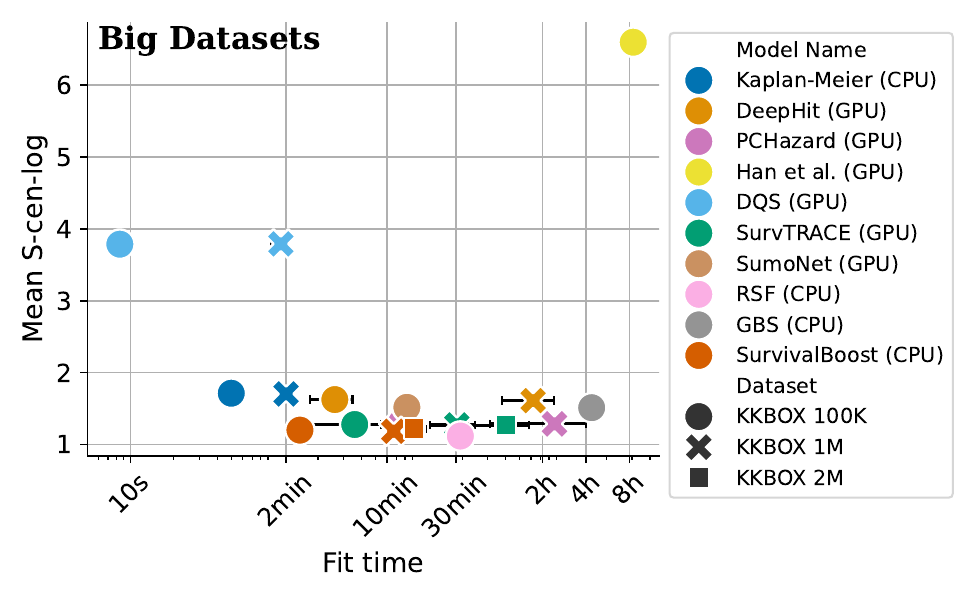}%
    \caption{\textbf{Trade-off between $S_{C-l-s}$ and fitting time for different sample sizes on the KKBOX dataset}}
    \label{fig:cenlog-kkbox}
\end{figure}

\subsection{Trade-off between training time and performances}\label{tradeoffyana}
Here, we provide the results of our analysis of training time with the performances on the $S_{Cen-log-simple}$ of the different models for the survival analysis.

\begin{figure}[h!]
    \centering
    \includegraphics[width=0.7\linewidth]{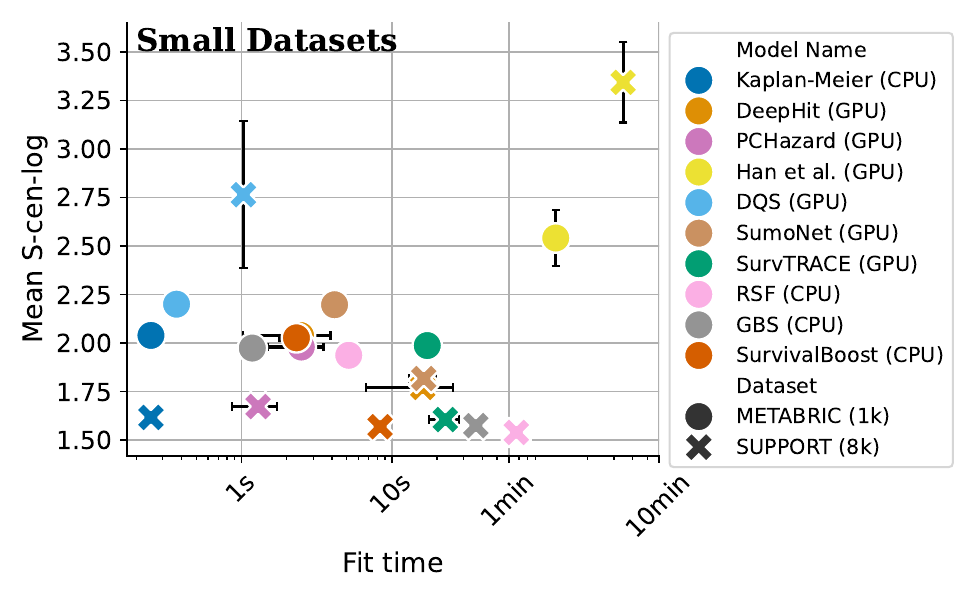}
    \caption{Trade-off between performance and the training time for the $S_{Cen-log-simple}$ metric for the survival model on METABRIC and SUPPORT datasets.}
    \label{fig:fit_predicttime}
\end{figure}

\subsection{Results for the SEER Dataset}%
\label{app:seer_results}

\paragraph{Learning curves}

We conducted experiments while varying the number of training points,  measuring the KM-adjusted Integrated Brier Score (IBS) for each event. Additionally, we averaged the scores to obtain a global metric. The IBS was computed for each event while training on the full dataset, except for Random Survival Forests, which was trained on 100k data points, and Fine and Gray, which was on 10k data points due to computational limitations. In Table \ref{tab:ibs_event_seer}, we compare our method with other models, showing that \textsc{SurvivalBoost} outperforms the alternatives. Furthermore, figure \ref{fig:tradeoff_competing} illustrates that the models with the best average IBS are also the fastest to train.

\begin{table}[h!]
\begin{center}
\caption{Integrated Brier Score for each cause-specific risk on the SEER Dataset (Lower is better). 
\label{tab:ibs_event_seer}}
\small\sc
\begin{tabular}{l|rrr}
\toprule
Event & 1 & 2 & 3 \\
\midrule
Aalen-Johansen & 0.1209 & 0.2832 & 0.0834\\
Fine \& Gray & 0.1055 & \underline{0.0281} & 0.0822 \\
Random Survival Forests& \textbf{0.0825}& 0.0295&0.0803 \\
DeepHit& 0.0931& 0.0330&0.0831 \\
DSM & 0.0875 & 0.0310 & 0.0869 \\
DeSurv & 0.0975 & 0.0327 & 0.0869 \\
SurvTRACE & 0.0871 & 0.0287 & \underline{0.0800} \\
\textsc{SurvivalBoost} & \underline{0.0832} & \textbf{0.0273} & \textbf{0.0757} \\
\bottomrule
\end{tabular}
\end{center}
\end{table}

\paragraph{$C_{\zeta}$-index}

The $C$-index measures whether the ranking of the risk for different
samples aligns with the order of the times when the
event of interest occurs\citep{harrell}. While it
was originally developed as a metric for survival analysis, it is often adapted to competing risks settings, where it is applied independently to each event \citep{uno_cstatistics_2011}.
However, in such settings, the C-index is biased and does not account for the probabilities of the events. Nonetheless, due to its popularity, we have included it in our experiments.

The tables below present the $C_\zeta$-index over time for the three
events \ref{tab:results_seer}. 
 At a fixed time horizon $\zeta$, we compute the $C_\zeta$-index for each class, which corresponds to the ROC-AUC, accounting for censored observations. The time horizons $\zeta$ are selected based on the any-event distribution, representing quantiles. For instance, at the time corresponding to 0.25, 25\% of the events have already occurred. 
These results differ from those in the SurvTRACE paper \citep{wang_survtrace_2022} for two main reasons: \emph{1)} The
available code online only implements one of their loss functions, \emph{2)} they treated the SEER dataset with two competing risks, classifying any other event as censored, whereas we categorized other events as a third competing risk.

\begin{table}[h!]
\begin{center}
\caption{C-index for competing risks on the SEER Dataset (Higher is better)
\label{tab:results_seer}}
\begin{small}
\adjustbox{max width =\columnwidth}{
\begin{tabular}{l||ccc||ccc||ccc}
\toprule
Time-horizon & & 0.25 & & & 0.50 & & & 0.75 & \\
quantile&&&&&&&&\\
\midrule
Event & 1 & 2 & 3 & 1 & 2 & 3 & 1 &2 &3 \\
\midrule
Aalen Johansen & .5±.0 & .5.±.0 & .5.±.0 & .5.±.0 & .5.±.0 & .5.±.0 & .5.±.0 & .5.±.0 & .5.±.0 \\
Fine \& Gray & .79.±.01 & .67.±.01 & .67.±.02 & .76.±.01 & .66.±.02 & .67.±.01 & .74.±.01 & .66.±.01 & .69.±.01 \\
DeepHit & .86.±.01 & .72.±.02 & .73.±.01 & .83.±.0 & .70.±.02 & .70.±.01 & .81.±.01 & .68.±.02 & .69.±.02 \\
DSM & \underline{.87.±.01} & \textbf{.76.±.01} & \underline{.74.±.01} & \underline{.84.±.01} & \textbf{.73.±.01} & \underline{.72.±.01} & \underline{.82.±.01} & \textbf{.72.±.01} & \textbf{.72.±.01} \\
DeSurv & .82.±.01 & .70.±.03 & .70.±.01 & .80.±.01 & .69.±.0 & .70± .01 & .79.±.01 & .68.±.01 & \underline{.71.±.01} \\
SurvTRACE & \textbf{.88.±.01} & \textbf{.76}.±.01 & \textbf{.76.±.01} & \textbf{ .85.±.01} & \textbf{.73.±.01} & \textbf{.73.±.01} & \textbf{.83.±.01} & \underline{.71.±.01} & \textbf{.72.±.01} \\
\bfseries SurvivalBoost & \underline{.87.±.01} & \underline{.75.±.01} & \underline {.74.±.01} & \underline{.84.±.01} & \underline{.72.±.01} & \underline{.72.±.01} & .80.±.01 & .64.±.01 & .62.±.01 \\
\bottomrule
\end{tabular}
}
\end{small}
\end{center}
\end{table}

\section{Additional results for survival experiments}

\subsection{Metrics for the survival analysis}\label{saall}

\begin{table}[h!]
\label{tab:metabric_supp}
\begin{center}
\begin{small}
\begin{sc}
\caption{\textbf{METABRIC}: $S_{Cen-log-simple}$ and C-index}
\begin{tabular}{l|rrrrr}
\toprule
Model Name & $S_{C-l-s}$ (↓) & C-index 0.25 (↑) &  C-index 0.5 (↑) & C-index 0.75 (↑) \\
\midrule
Kaplan-Meier &  2.0393 ± 0.2184 &  0.5000 ± 0.0000 &  0.5000 ± 0.0000 &  0.5000 ± 0.0000 \\             
DeepHit &  2.0391 ± 0.0005 &  0.6559 ± 0.0123 &  0.5918 ± 0.0236 &  0.6036 ± 0.0226 \\
PCHazard &  1.9796 ± 0.0855 &  0.6633 ± 0.0145 &  0.6356 ± 0.0112 &  0.6342 ± 0.0034 \\
Han et al. &  2.6648 ± 0.0356 &  \underline{0.6770 ± 0.0341} &  \textbf{0.6537 ± 0.0318} &  \textbf{0.6407 ± 0.0074} \\
DQS &  2.2002 ± 0.0000 &  0.6554 ± 0.0126 &  0.6215 ± 0.0091 &  0.6275 ± 0.0018 \\
SumoNet &  2.1973 ± 0.0000 &  \textbf{0.6872 ± 0.0230} &  0.6428 ± 0.0107 &  0.6292 ± 0.0084 \\
SurvTRACE &  1.9871 ± 0.0876 &  0.6598 ± 0.0094 &  0.6377 ± 0.0079 &  0.6357 ± 0.0108 \\
RSF &  \underline{1.9371 ± 0.2265} &  0.6736 ± 0.0135 &  0.6398 ± 0.0101 &  0.6335 ± 0.0097 \\
GBS &  \textbf{1.9742 ± 0.4043} &  0.6402 ± 0.0131 &  \underline{0.6399 ± 0.0122} &  \underline{0.6388 ± 0.0101} \\
\bfseries SurvivalBoost &  2.0269 ± 0.1592 &  0.6685 ± 0.0099 &  0.6374 ± 0.0106 &  0.6159 ± 0.0082 \\
\bottomrule
\end{tabular}
\end{sc}
\end{small}
\end{center}
\end{table}

\begin{table}[h!]
    \begin{center}
    \begin{small}
    \begin{sc}
    \caption{\textbf{METABRIC:} metrics.}
    \begin{tabular}{l|rrrr}
    \toprule
    Model Name &  IBS (↓) & MSE (↓) & MAE (↓) & AUC (↑) \\
    \midrule
    Kaplan-Meier &  0.1854 ± 0.0103 &  16007.1 ± 2100.4 &  102.3 ± 2.5 &  0.5000 ± 0.0000 \\
    DeepHit &  0.1707 ± 0.0086 &  16229.1 ± 1645.0 &   98.5 ± 2.3 &  0.6737 ± 0.0256 \\
    PCHazard &   0.1685 ± 0.011 &  15374.2 ± 2134.5 &   93.3 ± 3.2 &  0.6871 ± 0.0173 \\
    Han et al. &  0.1959 ± 0.0036 &  \textbf{13714.0 ± 1349.0} &   95.5 ± 2.4 &  0.6752 ± 0.0113 \\
    DQS &   0.1717 ± 0.018 &  16833.8 ± 1777.9 &   97.3 ± 2.5 &  0.6792 ± 0.0164 \\
    SumoNet &  0.1698 ± 0.0098 &  40239.2 ± 1936.9 &  179.8 ± 3.2 &  0.5000 ± 0.0000 \\
    SurvTRACE &  0.1723 ± 0.0064 &  22733.4 ± 1382.3 &  109.7 ± 4.7 &  0.6962 ± 0.0102 \\
    RSF &  \textbf{0.1651 ± 0.0084} &  15154.4 ± 1445.0 &   94.3 ± 1.2 &  \textbf{0.7023 ± 0.0129} \\
    GBS &  0.1686 ± 0.0107 &  14265.3 ± 2025.0 &   \underline{91.6 ± 3.4} &  0.6896 ± 0.0123 \\
    \bfseries SurvivalBoost &  \underline{0.1679 ± 0.0116} &  \underline{14208.1 ± 1762.8} &   \textbf{91.5 ± 2.7} &  \underline{0.6993 ± 0.0170} \\
    \bottomrule
\end{tabular}
    \label{tab:metabric-othermetrics}
    \end{sc}
\end{small}
\end{center}
\end{table}

\begin{table}[h!]
\label{tab:support_supp}
\begin{center}
\begin{small}
\begin{sc}
\caption{\textbf{SUPPORT}: $S_{Cen-log-simple}$ and C-index}
\begin{tabular}{l|rrrrr}
\toprule
Model Name & $S_{C-l-s}$ (↓) & C-index 0.25 (↑) &  C-index 0.5 (↑) & C-index 0.75 (↑) \\
\midrule
Kaplan-Meier &  1.6169 ± 0.2680 &  0.5000 ± 0.0000 &  0.5000 ± 0.0000 &  0.5000 ± 0.0000 \\
DeepHit &  2.249±.009 &  0.5546 ± 0.0158 &  0.5575 ± 0.0163 &  0.5600 ± 0.0196 \\
PCHazard &  1.6730 ± 0.0040 &  0.6121 ± 0.0052 &  0.6077 ± 0.0047 &  0.6054 ± 0.0044 \\
Han et al. &  3.2227 ± 0.0054 &  0.5920 ± 0.0235 &  0.5740 ± 0.0187 &  0.5713 ± 0.0143 \\
DQS &  2.7641 ± 0.1281 &  0.5741 ± 0.0043 &  0.5682 ± 0.0033 &  0.5645 ± 0.0038 \\
SumoNet &  1.8175 ± 0.0000 &  0.5948 ± 0.0050 &  0.5952 ± 0.0052 &  0.5970 ± 0.0050 \\
SurvTRACE &  1.6061 ± 0.0026 &  0.6101 ± 0.0052 &  0.6099 ± 0.0038 &  0.6073 ± 0.0030 \\
RSF&  1.9421 ± 0.0229 &  \textbf{0.6174 ± 0.0058} &  0.6137 ± 0.0045 &  0.6104 ± 0.0047 \\
GBS &  1.5750 ± 0.0002 &  0.6136 ± 0.0108 &  \underline{0.6140 ± 0.0100} &  \textbf{0.6143 ± 0.0099} \\
\bfseries SurvivalBoost &  \underline{1.5692 ± 0.3413} &  \underline{0.6165 ± 0.0052} &  \textbf{0.6159 ± 0.0044} &  \underline{0.6138 ± 0.0044} \\ 
\bottomrule
\end{tabular}
\end{sc}
\end{small}
\end{center}
\end{table}

\begin{table}[h!]
    \begin{center}
    \begin{small}
    \begin{sc}
    \caption{\textbf{SUPPORT:} metrics.}
    \begin{tabular}{l|rrrr}
    \toprule
    Model Name & IBS (↓) & MSE (↓) & MAE (↓) & AUC (↑)\\
    \midrule
    Kaplan-Meier & 0.2077 ± 0.004 & 1503075.2 ± 34398.0 & 904.4 ± 7.3 & 0.5000 ± 0.0000 \\
    DeepHit & 0.2061 ± 0.0058 & 1416882.2 ± 33011.4 & 898.4 ± 14.0 & 0.6061 ± 0.0321 \\
    PCHazard & 0.1867 ± 0.0036 & 1317674.8 ± 27353.9 & 843.0 ± 8.8 & 0.6578 ± 0.0074 \\
    Han et al. & 0.2539 ± 0.0015 & 1417630.2 ± 40832.6 & 881.5 ± 23.0 & 0.5906 ± 0.0139 \\
    DQS & 0.2025 ± 0.004 & 1499067.0 ± 44660.7 & 876.3 ± 11.1 & 0.5979 ± 0.0029 \\
    SumoNet & 0.1942 ± 0.0056 & 1857007.8 ± 36240.4 & 967.4 ± 8.1 & 0.5000 ± 0.0000 \\
    SurvTRACE & 0.1876 ± 0.0037 & 1294800.3 ± 14983.6 & 849.8 ± 12.5 & 0.6555 ± 0.0070 \\
    RSF & \underline{0.1815 ± 0.0041} & 1347923.5 ± 53819.8 & \underline{842.9 ± 15.2} & \textbf{0.6750 ± 0.0094} \\
    GBS & 0.187 ± 0.0041 & \underline{1292740.7 ± 26527.7} & 847.8 ± 10.0 & 0.6617 ± 0.0128\\
    \bfseries SurvivalBoost & \textbf{0.1814 ± 0.0049} & \textbf{1216995.5 ± 34370.6} & \textbf{827.2 ± 12.2} & \underline{0.6704 ± 0.0086} \\ 
    \bottomrule
\end{tabular}
    \label{tab:support-othermetrics}
    \end{sc}
\end{small}
\end{center}
\end{table}

% \begin{table}[h!]
% \label{tab:kkbox_supp}
% \begin{center}
% \begin{small}
% \begin{sc}
% \caption{\textbf{KKBOX (100k data points)}: $S_{Cen-log-simple}$ and C-index}
% \begin{tabular}{l|rrrrr}
% \toprule
% Model Name & $S_{C-l-s}$ (↓) & C-index 0.25 (↑) &  C-index 0.5 (↑) & C-index 0.75 (↑) \\
% \midrule

% \bottomrule
% \end{tabular}
% \end{sc}
% \end{small}
% \end{center}
% \end{table}

\begin{table}[h!]
    \begin{center}
    \begin{small}
    \begin{sc}
    \caption{\textbf{KKBOX (100k data points):} metrics.}
    \begin{tabular}{l|rrrr}
    \toprule
 Model Name & IBS (↓) & MSE (↓) & MAE (↓) & AUC (↑) \\
 \midrule
 Kaplan-Meier & 0.2131 ± 0.0007 & 177438.3 ± 2250.0 & 345.3 ± 1.2 & 0.5000 ± 0.0000 \\
DeepHit & 0.1523 ± 0.0007 & 113033.6 ± 593.1 & 245.7 ± 0.5 & 0.9397 ± 0.0052  \\
PCHazard & 0.1095 ± 0.0001 & \underline{100153.0 ± 1925.2} & 213.5 ± 3.2 & \underline{0.9431 ± 0.0046} \\
Han et al. (NLL) & 0.1715 ± 0.0036 & 111820.2 ± 0.0 & 245.6 ± 0.0 & 0.8881 ± 0.0086 \\
DQS & 0.1301 ± 0.0013 & \textbf{93820.2 ± 4140.5} & \textbf{204.9 ± 2.2} & 0.9228 ± 0.0071 \\
SumoNet & 0.1078 ± 0.0 & 224981.4 ± 0.0 & 360.3 ± 0.0 & 0.5000 ± 0.0000 \\
SurvTRACE & 0.1107 ± 0.0006 & 133400.5 ± 1353.9 & 250.0 ± 3.0 & 0.9379 ± 0.0004 \\
RSF & \underline{0.1068 ± 0.0} & 911586.7 ± 0.0 & 423.6 ± 0.0 & \textbf{0.9449 ± 0.0000} \\
GBS & 0.1567 ± 0.0 & 123348.9 ± 0.0 & 254.5 ± 0.0 & 0.8958 ± 0.0000 \\
\bfseries SurvivalBoost & \textbf{0.1052 ± 0.0006} & 101103.9 ± 9688.4 & \underline{207.2 ± 4.3} & 0.9322 ± 0.0006 \\
\bottomrule
\end{tabular}
    \label{tab:kkbox-othermetrics}
    \end{sc}
\end{small}
\end{center}
\end{table}

\begin{table}[h!]
    \begin{center}
    \begin{small}
    \begin{sc}
    \caption{\textbf{Survival Calibration Metrics.} results are marked with \vmark \, if the model is calibrated and - \, otherwise, with the significance level fixed at $p_{value} = 0.05$.}
    \begin{tabular}{l|rrrr|rrrr|rrrr|r}
        \toprule
        \hfill Dataset & \multicolumn{4}{c|}{\textbf{METABRIC}} & \multicolumn{4}{c|}{\textbf{SUPPORT}} & \multicolumn{4}{c|}{\textbf{KKBOX}} & \textbf{Total} \\ 
            & $\textsc{KM}_c$ & $\textsc{X}_c$ & $\textsc{D}_c$ & $\textsc{one}_c$ &  $\textsc{KM}_c$ & $\textsc{X}_c$ & $\textsc{D}_c$ & $\textsc{one}_c$ & $\textsc{KM}_c$ & $\textsc{X}_c$ & $\textsc{D}_c$ & $\textsc{one}_c$ &\textbf{tests} \\
          Model&&&&&&&&&&&&&\textbf{succesfull} \\
          \midrule
          Kaplan-Meier &  \vmark  & \vmark & \vmark & \vmark & \vmark & \vmark &  \vmark & - &\vmark & \vmark & - & \vmark & \bf 10 \\
          DeepHit & \vmark  & \vmark & -  & - & \vmark & - & - & - & \vmark  & \vmark & -  & - & 5 \\
          PCHazard  & \vmark  & \vmark & \vmark & \vmark & \vmark & \vmark & - &- & \vmark & \vmark & - & - & 8\\
          \cite{han2021inverse}  & \vmark  & \vmark & -  & \vmark & \vmark & \vmark & - & - & \vmark & - & - & - & 6\\
          DQS & \vmark  & \vmark & -  & - & \vmark & - & \vmark & - &   \vmark & - & - & - & 5 \\
          SumoNet & - & - & - & - & - & - & - & - & - & - & - & - & 0\\
          SurvTRACE  & -  & \vmark  & \vmark  & \vmark & \vmark & \vmark & - & - & \vmark & \vmark & - & -  & 7\\
          RSF   & \vmark  & \vmark & \vmark  & \vmark & \vmark & \vmark & \vmark & - & \vmark & \vmark & - & -  & \bf 9 \\
          GBS & \vmark & \vmark & \vmark  & \vmark &  \vmark & \vmark & - & - & \vmark & \vmark & - & - & 8\\
          \bfseries \textsc{SurvivalBoost} & \vmark  & \vmark & \vmark  & \vmark & \vmark & \vmark & \vmark &  - & \vmark & \vmark & - & - & \bf 9\\
          \bottomrule
    \end{tabular}
    \label{tab:calibration-metrics}
    \end{sc}
\end{small}
\end{center}
\end{table}

\section{Implementation Details}
\subsection{Computing Infrastructure}\label{infra}
To conduct our experiments, we used an iternal cluster.  \\
The neural network were trained onto a 4x NVIDIA Tesla V100 32GB GPU with 40 CPUs and 252Gb RAM. \\
The others methods that do not need GPUs were trained onto a cluster with 48CPUs and 504Gb RAM. We chose to allow only 50Gb RAM for each model.

\subsection{Reference of used implementations for baselines}

We compare \textsc{SurvivalBoost} with several baselines, outlining their main characteristics and the implementation used in Table~\ref{tab:implem}

\begin{table}[h!]
\centering
\caption{Characteristics of used baselines.}
\begin{tabular}{|l|c|c|p{5cm}|p{3cm}|} \toprule
Name & \makecell{Competing\\ risks} & \makecell{Proper\\loss} & Implementation & Reference \\
\midrule
SurvTRACE & \checkmark & & ours & \citet{wang_survtrace_2022} \\ \midrule
DeepHit & \checkmark & & \rurl{github.com/havakv/pycox} & \citet{lee_deephit_2018}\\\midrule
DSM & \checkmark & & \rurl{autonlab.github.io/DeepSurvivalMachines} & \citet{nagpal2021deep}\\\midrule
DeSurv & \checkmark & & \rurl{github.com/djdanks/DeSurv}& \citet{danks_derivative-based_2022} \\\midrule
\makecell{Random Survival\\Forests} & \checkmark & & \rurl{scikit-survival.readthedocs.io/} for survival, and \rurl{www.randomforestsrc.org/} for competing risks & \citet{ishwaran_random_2008,ishwaran2014random} \\\midrule
Fine \& Gray & \checkmark & & \rurl{cran.r-project.org/package=cmprsk} & \citet{fine_proportional_1999}\\\midrule
Aalen-Johansen & \checkmark & & ours & \citet{aalen_survival_2008} \\\midrule
Han et al. & & & \rurl{github.com/rajesh-lab/Inverse-Weighted-Survival-Games} & \citet{han2021inverse}\\\midrule
PCHazard & & & \rurl{github.com/havakv/pycox} & \citet{kvamme_continuous_2019}\\\midrule
SumoNet & & \checkmark & \rurl{github.com/MrHuff/Sumo-Net} & \citet{rindt_survival_2022} \\\midrule
DQS & & \checkmark & \rurl{ibm.github.io/dqs/} & \citet{yanagisawa2023proper}\\
\bottomrule
\end{tabular}
\label{tab:implem}
\end{table}

\subsection{GridSearch Parameters}
We performed a Randomized Search for these parameters with a budget of 30 iterations. There are no parameters to tune for the Aalen-Johansen and Fine \& Gray models.
\begin{table}[h!]
    \centering
    \caption{Randomized Search Parameters}
    \begin{tabular}{|l|l||c|}
    \toprule
    Estimator & Parameter & Range \\
    \midrule
    \textsc{SurvivalBoost} & Learning Rate & $loguniform(0.01, 0.5)$ \\
    & Nb of iterations & $\llbracket 20, 200 \rrbracket$ \\
    & Maximum Depth & $\llbracket 2, 10 \rrbracket$ \\
    & Nb of times & $\llbracket 1, 5 \rrbracket$ \\
    \midrule
    SurvTRACE & Learning Rate & $loguniform(10^{-5}, 10^{-3})$\\
    & Batch Size & $\{256, 512, 1024\}$\\
    & Hidden parameter & $\{2, 3\}$ \\ 
    \bottomrule
    \end{tabular}
    \label{tab:gridsearch}
\end{table}

\section{Distribution of the competing risks datasets}
\subsection{SEER Distribution of events}
Here, we present the distributions for both competing risks datasets: the SEER Dataset \ref{fig:seer} and the synthetic dataset \ref{fig:distrib_synthe}. Notably, the censoring distribution is non-uniform over time. Figure \ref{fig:distrib_synthe} illustrates an example of the event distribution with censoring, which is dependent on the covariates. The parameters were selected to represent three distinct behaviors. 

\begin{figure}[h]
  \centering
  \begin{minipage}[b]{0.49\textwidth}
    \includegraphics[width=\columnwidth]{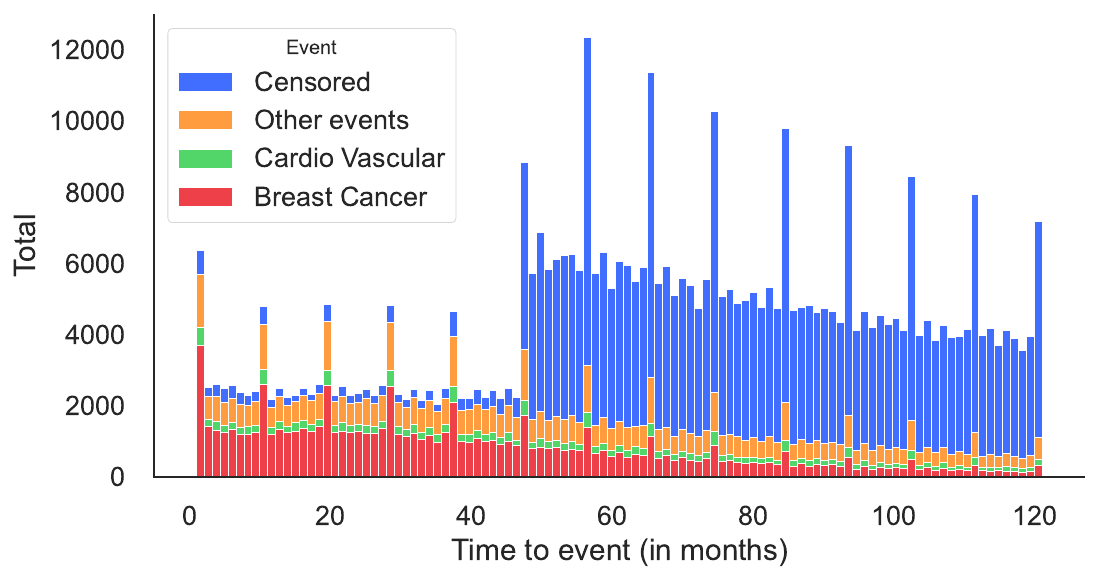}
    \caption{\textbf{SEER Dataset Distributions} The censoring rate is approximately 63\%. The prevalence of events is 18\% for breast cancer, 4.5\% for cardiovascular events, and 10\% for other events. The shift in the censoring distribution after the $48^{th}$ month may be challenging for some methods to learn.}
    \label{fig:seer}
  \end{minipage}
  \hfill
  \begin{minipage}[b]{0.49\textwidth}
    \includegraphics[width=\columnwidth]{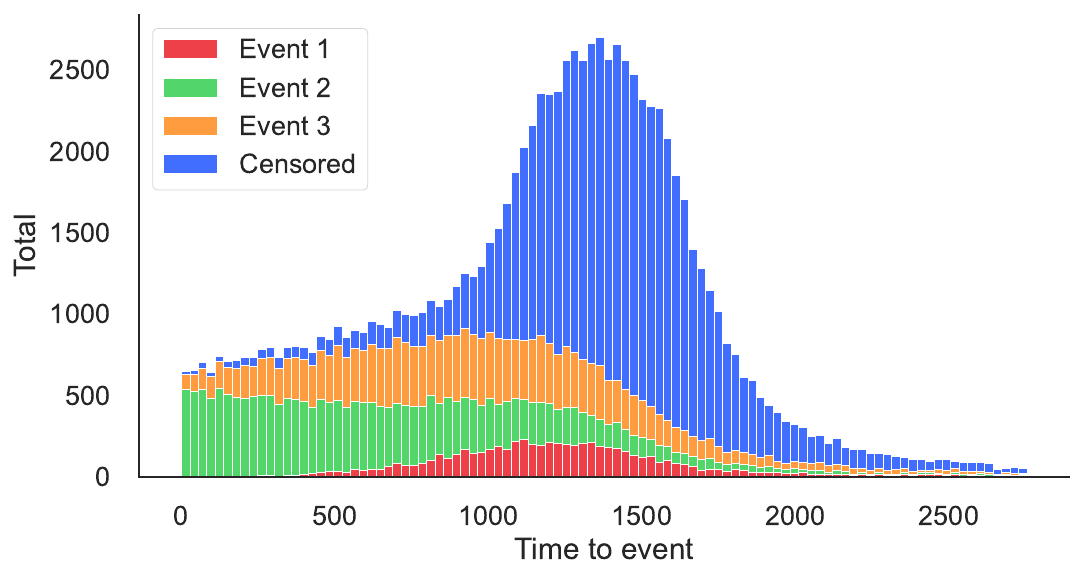}
    \caption{\textbf{Synthetic Dataset Distributions} Duration distributions of the synthetic dataset with censoring dependent on X, and a censoring rate of 69\%. The events are stacked. To illustrate this distribution, consider truck maintenance. Event 1, occurring throughout the duration, corresponds to drivers' driving skills. Event 2 may represent a design flaw in the trucks, occurring from the start. Event 3 refers to trucks' wear and tear over time. }
\label{fig:distrib_synthe}
  \end{minipage}
\end{figure}

%%%%%%%%%%%%%%%
\end{document}